\documentclass[letterpaper]{article}
\pdfoutput=1 
\usepackage[a4paper, total={6in,8in}]{geometry}
\usepackage{times}  
\usepackage{helvet}  
\usepackage{courier}  
\usepackage[hyphens]{url}  
\usepackage{graphicx} 
\usepackage{natbib}  
\usepackage{caption} 

\usepackage[utf8]{inputenc} 
\usepackage{url}            
\usepackage{amsfonts}       
\usepackage{nicefrac}       
\usepackage{microtype}      
\usepackage{kbordermatrix}
\usepackage{multirow}
\usepackage[table]{xcolor}  

\usepackage{amsmath}
\usepackage{amsthm}
\usepackage{amssymb}
\usepackage{bbm}
\usepackage{paralist}
\usepackage{enumitem}
\usepackage{booktabs,floatrow}
\usepackage{mathtools}
\usepackage{dsfont}
\usepackage{natbib}

\usepackage{blindtext}
\usepackage{graphicx}
\usepackage[utf8]{inputenc}
\usepackage{listings}
\usepackage{bbm}
\usepackage{amssymb}

\usepackage{amsthm}
\usepackage{amsmath}
\usepackage{algorithm,algorithmic}
\DeclareMathOperator*{\argmax}{arg\,max}
\DeclareMathOperator*{\argmin}{arg\,min}
\theoremstyle{definition}
\usepackage{mathtools}

\usepackage{amsmath}
\usepackage{tikz}
\usepackage{mathdots}
\usepackage{yhmath}
\usepackage{cancel}

\usepackage{siunitx}
\usepackage{array}
\usepackage{multirow}
\usepackage{amssymb}
\usepackage{gensymb}
\usepackage{tabularx}
\usepackage{subcaption}
\usetikzlibrary{fadings}
\usetikzlibrary{patterns}
\usetikzlibrary{shadows.blur}
\usetikzlibrary{shapes}
\usepackage{hyperref}
\hypersetup{colorlinks,
            linkcolor=blue,
            citecolor=blue,
            urlcolor=magenta,
            linktocpage,
            plainpages=false}
\usepackage[capitalise]{cleveref}
\renewcommand{\cite}{\citep}

\newcommand{\kl}{\textnormal{kl}}

\newtheorem{theorem}{Theorem}

\newtheorem{lemma}{Lemma}
 
\newtheorem{proposition}{Proposition}
\newtheorem{corollary}{Corollary}

\newcommand{\set}[1]{\mathcal{#1}}
\newcommand{\ep}{\hfill $\Box$}

\newcommand{\indi}{\mathbbm{1}}
\newcommand{\skl}{\textnormal{kl}}

\title{On the Sample Complexity of Representation Learning in Multi-task Bandits with Global and Local structure}
\author{%
Alessio Russo, Alexandre Proutiere
}

\begin{document}

\maketitle
\begin{abstract}
We investigate the sample complexity of learning the optimal arm for multi-task bandit problems. Arms consist of two components: one that is shared across tasks (that we call representation) and one that is task-specific (that we call predictor). The objective is to learn the optimal (representation, predictor)-pair for each task, under the assumption that the optimal representation is common to all tasks. Within this framework, efficient learning algorithms should transfer knowledge across tasks. We consider the best-arm identification problem for a fixed confidence, where, in each round, the learner actively selects both a task, and an arm, and observes the corresponding reward. We derive instance-specific sample complexity lower bounds satisfied by any $(\delta_G,\delta_H)$-PAC algorithm (such an algorithm identifies the best representation with probability at least $1-\delta_G$, and the best predictor for a task with probability at least $1-\delta_H$). We devise an algorithm \textsc{OSRL-SC} whose sample complexity approaches the lower bound, and scales at most as $H(G\log(1/\delta_G)+ X\log(1/\delta_H))$,   with $X,G,H$ being, respectively, the number of tasks, representations and predictors. By comparison, this scaling is significantly better than the classical  best-arm identification algorithm that scales as $HGX\log(1/\delta)$. The code can be found here \url{https://github.com/rssalessio/OSRL-SC}.

\end{abstract}


\section{Introduction}\label{sec1}

Learning from previous tasks and transferring this knowledge may significantly improve the process of learning new tasks. This idea, at the core of transfer learning \cite{pan2009survey,skinner1965science,woodworth1901influence}, lifelong learning \cite{thrun1996learning} and multi-task learning \cite{baxter2000model,caruana1995learning,caruana1997multitask}, has recently triggered considerable research efforts with applications in both supervised and reinforcement learning. Previous work on transfer and multi-task learning has mostly focused on batch learning problems \cite{lazaric2012transfer,pan2009survey}, where when a task needs to be solved, a training dataset is directly provided. Online learning problems, where samples for a given task are presented to the learner sequentially, have been much less studied \cite{taylor2009transfer, zhan2015online}. 
In this paper, we consider a multi-task Multi-Armed Bandit (MAB) problem, where the objective is to find the optimal arm for each task using the fewest number of samples, while allowing to transfer knowledge across tasks. The problem is modelled as follows: in each round, the learner actively selects a task, and then selects an arm from a finite set of arms. Upon selecting an arm, the learner observes a random reward from an unknown distribution that represents the performance of her action in that particular task.
To allow the transfer of knowledge across the various tasks, we study the problem for a simple, albeit useful model. We assume that the arms available to the learner consist of two components: one that is shared across tasks (that we call representation) and one that is task-specific (that we call predictor). Importantly, the optimal arms for the various tasks share the same representation. The benefit of using this model is that we can study the sample complexity of learning the best shared representation across tasks while learning the task-specific best action.
Contribution-wise, in this work we derive instance-specific sample complexity lower bounds satisfied by any $(\delta_G,\delta_H)$-PAC algorithm (such an algorithm identifies the best representation with probability at least $1-\delta_G$, and the best predictors with probability at least $1-\delta_H$). Again, our lower bounds can be decomposed into two components, one for learning the representation, and one for learning the predictors. We devise an algorithm, \textsc{OSRL-SC}, whose sample complexity approaches the lower bound, and scales at most as $H(G\log(1/\delta_G)+ X\log(1/\delta_H))$. Finally, we present numerical experiments to illustrate the gains in terms of  sample complexity one may achieve by transferring knowledge across tasks.

\paragraph{Related work.} Multi-task learning has been investigated under different assumptions on the way the learner interacts with tasks. One setting concerns batch learning (often referred to as learning-to-learn), where the training datasets for all tasks are available at the beginning \cite{baxter2000model,maurer2005algorithmic,maurer2009transfer,maurer2013sparse}. The so-called batch-within-online setting considers that tasks arrive sequentially, but as soon as a task arrives, all its training samples are available \cite{balcan2015efficient,pentina2015multi,pentina2016lifelong,alquier2016regret}. Next, in the online multi-task learning  \cite{agarwal2008matrix,abernethy2007multitask,dekel2007online,cavallanti2010linear,saha2011online, lugosi2009online,murugesan2016adaptive,yang2020impact}, in each round, the learner observes a new sample for each task, which, in some cases, this may not be possible. Our framework is different as in each round the learner  can only select a single task. This framework has also been considered in \cite{lazaric2013sequential,soare2014multi,soare2015sequential,alquier2016regret,wu2019lifelong}, but typically there, the learner faces the same task for numerous consecutive rounds, and she cannot select the sequence of tasks. Also, note that the structure tying the reward functions of the various tasks together is different from ours. The structure tying the rewards of actions for various contexts is typically linear, and it is commonly assumed that there exists a common low-dimensional representation, or latent structure, to be exploited \cite{soare2014multi,soare2015sequential,deshmukh2017multi,kveton2017stochastic,hao2019adaptive,lale2019stochastic,yang2020impact,lu2021low}, or that the reward is smooth across tasks and/or arms \cite{magureanu2014lipschitz,slivkins2014contextual}. The aforementioned papers address scenarios where the context sequence is not controlled, and investigate regret.  Meta-learning is also closely connected to meta-learning \cite{cella2020meta,kveton2021meta,azizi2022meta}. In \cite{azizi2022meta} the authors investigate  the problem of simple regret minimization in a fixed-horizon when tasks are sampled i.i.d. from some unknown distribution. To our knowledge, this paper is the first to study how tasks should be scheduled towards a sample-optimal instance-specific best-arm identification algorithm.

\section{Model and assumptions}\label{sec3}
In this section, we describe the analytical model of the multi-task MAB problem considered, and describe the framework of best-arm identification for this class of multi-task MAB models.
\begin{figure}[t]
	\centering
	\includegraphics[width=0.5\linewidth]{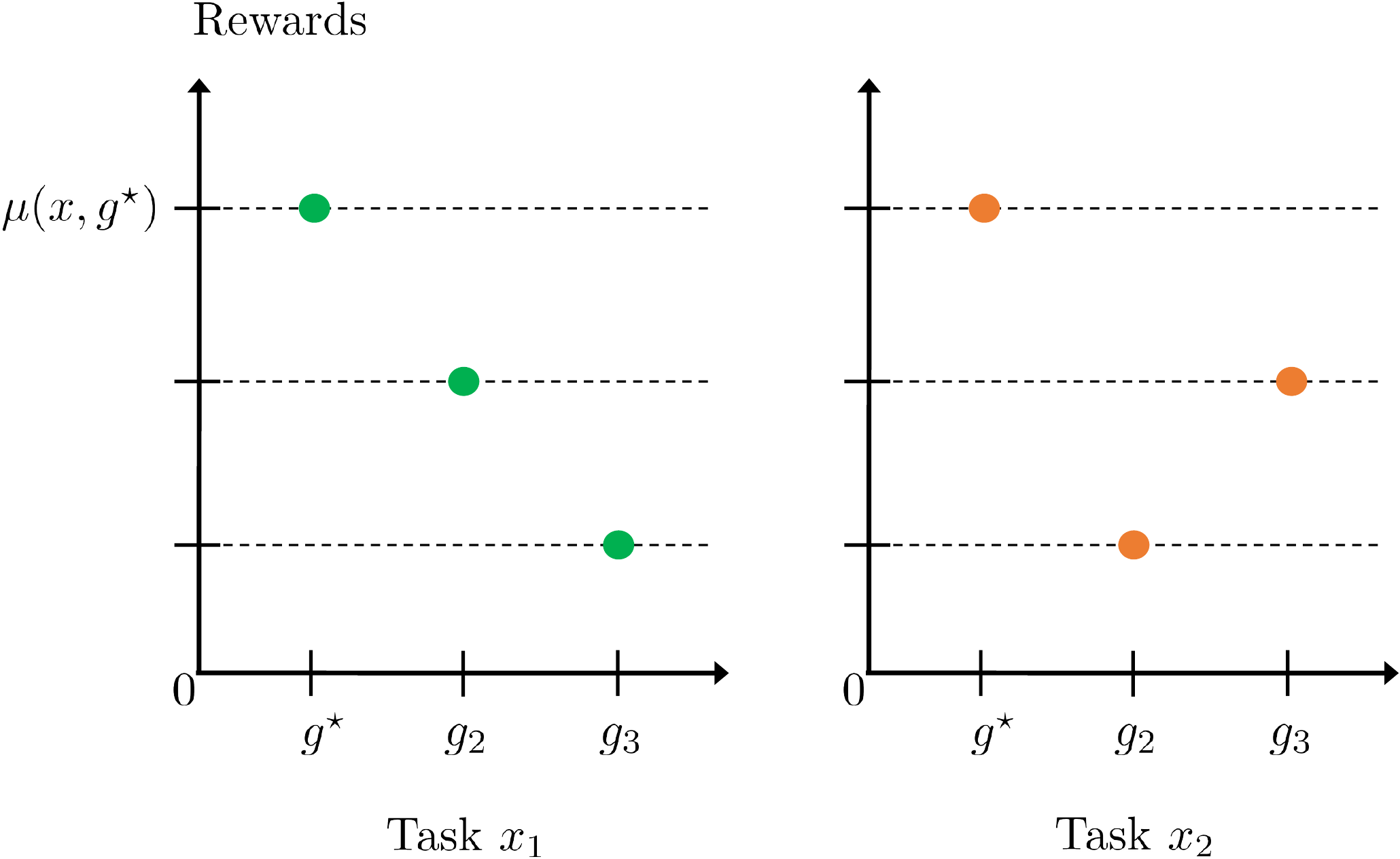}
	\caption{Example of two symmetric tasks $x_1$ and $x_2$, where learning the optimal representation $g^\star$ can be accelerated by considering both tasks, instead of focusing only on a single task. Task $x_1$ can be used to learn that $g_3$ is suboptimal, while task $x_2$ can be used to learn that $g_2$ is suboptimal.}
	\label{fig:symmetric_tasks}
\end{figure}
\paragraph{Model.} We consider multi-task MAB problems with  a finite set ${\cal X}$ of $X$ tasks. In each round $t\ge 1$, the learner chooses a  task $x\in {\cal X}$ and an arm $(g,h)\in {\cal G}\times {\cal H}$. The components $g$ and $h$ are, respectively, referred to as the {\it representation} and the {\it predictor}. When in round $t$, $x(t)=x$ and the learner selects $(g,h)$, she collects a binary reward $Z_{t}(x,g,h)$ of mean $\mu(x,g,h)$ (for the sake of the analysis we only analyse the binary case, although it can be easily extended to the Gaussian case as in \cite{garivier2016optimal}).  The rewards are $i.i.d.$ across rounds, tasks, and arms. Consequently, the system is characterized by $\mu=(\mu(x,g,h))_{x,g,h}$, which is  unknown to the learner.
 The main assumption made throughout the paper is that tasks share a common optimal representation $g^\star$: for any task $x\in {\cal X}$, there is a predictor $h_x^\star$ such that $(g^\star,h_x^\star)$ yields an optimal reward. Formally,  
\begin{equation}\label{eq:dominate}
\forall x\in {\cal X},\quad \mu(x,g^\star,h^\star_x)> \max_{(g,h)\neq(g^\star,h_x^\star)}\mu(x,g,h).
\end{equation}
 Moreover, note that there is no assumption on the smoothness of $\mu$ with respect to $(x,g,h)$.
This type of model represents the case where a learner can actively choose the task to execute (as if a generative model is available to the learner), and in this way maximize the learning efficiency by accurately picking tasks to reduce the sample complexity. Since the model is quite generic, it can be applied to a variety of problems where a collection of tasks have a local component, and a shared global component: (i) 
influence mechanisms with global/local groups; (ii) hyperparameters learning across multiple tasks; (iii) for advanced clinical trials, where, depending on the group of patients (tasks, that vary according to factors such as age, severity of the disease, etc.), different drugs and dosages can be used for inoculation ($g$ and $h$).

\paragraph{Sample complexity.} The objective of the learner is to devise a policy  that learns the best arms $(g^\star,h_1^\star,\ldots,h_X^\star)$ with the least number of samples.  Here, a policy $\pi$ is defined as follows. Let ${\cal F}_t^\pi$ denote the $\sigma$-algebra generated by the observations made under $\pi$ up to and including round $t$. Then $\pi$ consists of (i) a sampling rule: in each round $t$, $\pi$ selects a ${\cal F}_{t-1}$-measurable task $x^\pi(t)$ and an arm $(g^\pi(t),h^\pi(t))$; (ii) a stopping rule defined by $\tau$, which is a stopping time w.r.t. the filtration $({\cal F}_t)_{t\ge 1}$; (iii) a decision rule returning a ${\cal F}_{\tau}$-measurable estimated best arm for each task $(\hat{g}, \hat{h}_1,\ldots,\hat{h}_X)$.  Then, the performance of a policy $\pi$ is assessed through its PAC guarantees and its expected sample complexity $\mathbb{E}[\tau]$. PAC guarantees concern both learning $g^\star$ and $({h}_1^\star,\ldots,{h}_X^\star)$. Denote by ${\cal M}=\{\mu : \exists (g^\star,h_1^\star,\ldots,h_X^\star) :  \hbox{\cref{eq:dominate} holds} \}$ the set of systems where tasks share a common optimal  representation. Then, we say that $\pi$ is $(\delta_G,\delta_H)$-PAC if for all $\mu\in {\cal M}$, 
\begin{align}
    &\mathbb{P}_{\mu}(\tau<\infty)=1, \ \ \mathbb{P}_{\mu}\left(\hat{g} \neq g^\star\right) \leq \delta_G,\hbox{ and }\\
    &\mathbb{P}_{\mu}\left(\hat{h}_{x}\neq h_{x}^\star, \hat{g}=g^\star \right) \leq \delta_H, \quad \forall x \in \mathcal{X}.
\end{align}




\section{Sample complexity lower bound and the \textsc{OSRL-SC} algorithm}\label{sec5}
This section is devoted to the best-arm identification problem for the model considered in this work. We first derive a lower bound for the expected sample complexity of any  $(\delta_G,\delta_H)$-PAC algorithm, and then present an algorithm approaching this limit. In what follows, we let $\boldsymbol{\delta}=(\delta_G,\delta_H)$.
\subsection{Sample complexity lower bound}
We begin by illustrating the intuition behind the sample complexity lower bound, and then state the lower bound theorem. 
To identify the optimal representation $g^\star$ in a  task as quickly as possible, an algorithm should be able to perform \textit{information refactoring}, i.e., to use all the available information across tasks to estimate $g^\star$. To illustrate this concept, we use the model illustrated in  \cref{fig:symmetric_tasks}. In this model, there are only $2$ tasks $x_1,x_2$,  and $3$ arms in $\set{G}$ (and only $1$ arm in $\set{H}$, thus it can be ignored). For this model, to learn that $g_3$ is sub-optimal, we should mainly sample task $x_1$, since the gap between the rewards of $g^\star$ and $g_3$ is the largest. Similarly, learning that $g_2$ is sub-optimal should be mainly obtained by sampling task $x_2$. Using the same task, to infer that $g_2$ and $g_3$ are sub-optimal, would be much less efficient. This observation also motivates why it is inefficient to consider tasks separately, even in the case where $\mu$ is highly non-smooth with respect to $(x,g,h)$, and also motivates the expression of the sample complexity lower bound that we now present.

\paragraph{Sample complexity lower bound.} 
Computing the sample complexity lower bound amounts to finding the lower bound of a statistical hypothesis testing problem, which is usually done by finding what is the most confusing model. In this case the lower bound is given by the solution of the following optimization problem.
\begin{theorem}\label{thm:full_lower_bound_offline}
The sample complexity $\tau_{\boldsymbol{\delta}}$ of any $\boldsymbol{\delta}$-PAC algorithm satisfies:  $\mathbb{E}_\mu[\tau_{\boldsymbol{\delta}}] \geq K^{\star}(\mu, \boldsymbol{\delta})$ for any $\mu\in \set{M}$, where $K^{\star}(\mu, \boldsymbol{\delta})$ is the value of the optimization problem\footnote{Refer to the appendix for all the proofs.}:
\begin{align}
\min_{\eta}\quad & \sum_{x,g,h}\eta (x,g,h)\\
\textrm{s.t.} \quad &  \min_{h\neq h_x^{\star}} f_\mu(\eta,x,h)\geq \skl(\delta_H, 1-\delta_H)\quad \forall x,\label{eq:cons1}\\
 &\min_{\bar g\neq g^{\star}} \sum_{x}  \min_{\bar h_x}\ell_\mu(\eta,\bar g, \bar h_x) \geq \skl(\delta_G, 1-\delta_G).\label{eq:cons2}
\end{align}
\end{theorem}
\noindent where  in the first constraint $f_\mu(\eta,x,h) =(\eta (x,g^{\star},h_x^{\star})+\eta (x,g^{\star},h) )I_{\alpha_{x,g^{\star},h}}(\mu(x,g^{\star},h_x^{\star}), \mu(x,g^{\star},h))$ accounts for the difficulty of learning the best predictor $h_x^\star$ for each task $x$. The term $I_\alpha (\mu_1,\mu_2)=\alpha \skl (\mu_1, d_\alpha(\mu_1,\mu_2))+(1-\alpha) \skl (\mu_2, d_\alpha(\mu_1,\mu_2))$ is a generalization of  the Jensen-Shannon divergence, with $d_\alpha(\mu_1,\mu_2)=\alpha\mu_1+(1-\alpha) \mu_2$, and $ \alpha\in[0,1]$. Finally, the term $\alpha_{x,g,h} \coloneqq \eta (x,g^{\star},h_x^{\star})/(\eta (x,g^{\star},h_x^{\star})+\eta (x,g,h))$  represents the proportion of time  $(x,g^\star, h_x^\star)$ is picked over $(x,g,h)$.

The second constraint, $\ell_\mu(\eta,\bar g,\bar h_x)$ accounts for the difficulty of learning the optimal $g^\star$. To define it, let the average reward over $\mathcal{C} \subseteq \mathcal{G}\times\mathcal{H}$ for some task $x$ and allocation $\eta$ to be defined as 
\begin{equation}
    m(x,\eta,\mathcal{C}) = \frac{\sum_{(g,h)\in \mathcal{C}} \eta(x,g,h)\mu(x,g,h)}{\sum_{(g,h)\in \mathcal{C}} \eta(x,g,h)}.
\end{equation}
Then, $\ell_\mu(\eta,\bar g,\bar h_x)$  is given by:
\small
\begin{equation}
     \ell_\mu(\eta,\bar g, \bar h_x)=\smashoperator{\sum_{(g,h)\in \set{U}_{x,\bar g, \bar h_x}^{\eta,\mu}}} \eta(x,g,h) \skl\left(\mu(x,g,h), m\left(x,\eta, \mathcal{U}_{x,\bar g, \bar h_x}^{\eta,\mu}\right)\right),
\end{equation}\normalsize
where $\skl(a,b)$ is the KL divergence between two Bernoulli distributions of respective means $a$ and $b$, and  $ \set{U}_{x,\bar g, \bar h_x}^\mu$   is the set of confusing arms for task $x$.  Then $m\left(x,\eta, \mathcal{U}_{x,\bar g, \bar h_x}^{\eta,\mu}\right)$ represents the average reward of the confusing model for estimating $g^\star$.

The set $\set{U}_{x,\bar g, \bar h_x}^{\eta,\mu}$ is defined by through $\mathcal{N}_{x,g,h; \bar g, \bar h_x}^\mu$, which is the set of arms whose mean is bigger than $\mu(x,g,h)$ and that also include $(\bar g, \bar h_x)$, i.e.
$
{\cal N}_{x,g,h; \bar g, \bar h_x}^\mu= \{(g',h') : \mu(x,g',h')\geq \mu(x,g,h) \} \cup \{(\bar g, \bar h_x)\}
$.

Then, the set of confusing arms $\set{U}_{x,\bar g, \bar h_x}^{\eta,\mu} $ is given by 
\begin{align*}
    \set{U}_{x,\bar g, \bar h_x}^{\eta,\mu} = &\left\{(g,h) :  \mu(x,g,h) \geq m\left(x,\eta, {\cal N}_{x,g,h; \bar g, \bar h_x}^\mu\right) \right\}\\&\qquad\cup \{(\bar g, \bar h_x)\}.
\end{align*}


\begin{figure}[t]
	\centering
	\includegraphics[width=0.7\linewidth]{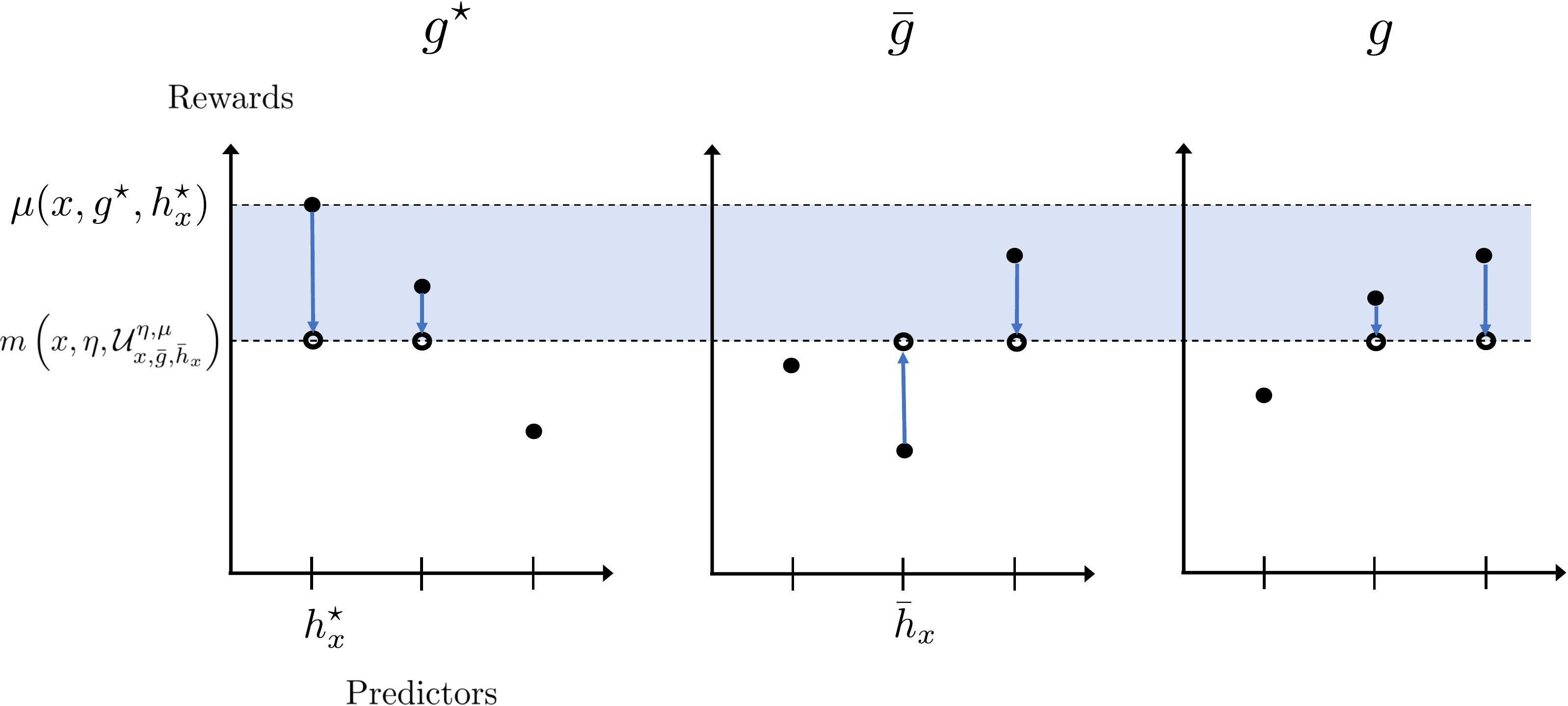}
	\caption{Example of the set  $\set{U}_{x,\bar g, \bar h_x}^{\eta,\mu}$. All points in the shadowed area are in the set. All arms with average reward above $m\left(x,\eta, \mathcal{U}_{x,\bar g, \bar h_x}^{\eta,\mu}\right)$  belong to $\set{U}_{x,\bar g, \bar h_x}^{\eta,\mu}$ (also $(\bar{g},\bar{h}_x)\in \set{U}_{x,\bar g, \bar h_x}^{\eta,\mu}$).}
	\label{fig:computation_confusing_problem}
\end{figure}
In \cref{thm:full_lower_bound_offline} $\eta(x,g,h)$ can be interpreted as the minimal expected number of times any $\boldsymbol{\delta}$-PAC algorithm should select $(x,g,h)$.  \cref{eq:cons1} corresponds to the constraints on $\eta$ one has to impose so that the algorithm learns the optimal predictor $h_x^\star$ for all $x$, while \cref{eq:cons2} is needed to ensure that the algorithm identifies the best representation $g^\star$ across all tasks. Both \cref{eq:cons1} and \cref{eq:cons2} define two convex sets in terms of $\eta$, and hence $K^{\star}(\mu, \boldsymbol{\delta})$ is the value of a convex program.

Observe that the constraints \cref{eq:cons1} and \cref{eq:cons2} share the components of $\eta$ that concern $g^\star$ only. We believe that actually separating the problem into two problems, one for each constraint, as formulated in the proposition below yields a tight upper bound of $K^\star(\mu,\boldsymbol{\delta})$.

\begin{proposition}\label{pro:1}
We have $K^\star(\mu,\boldsymbol{\delta})\le K_H^\star(\mu,\delta_H) + K_G^\star(\mu,\delta_G)$, where $K_H^\star(\mu,\delta_H)$ (resp. $K_G^\star(\mu,\delta_G)$) is the value of the problem:   $\min_{\eta\geq 0} \sum_{x,g,h}\eta(x,g,h)$ subject to (\cref{eq:cons1}) (resp. (\cref{eq:cons2})). 
\end{proposition}

Note that $K_H^\star(\mu,\delta_H)$ scales as $HX \kl(\delta_H,1-\delta_H)$ (since the corresponding optimization problem is that obtained in a regular bandit problem for each task \cite{garivier2016optimal}, which scales as $H\kl(\delta_H, 1-\delta_H)$ for each task). Now, to know how $K_G^\star(\mu,\delta_G)$ scales, we can further derive an upper bound of $K_G^\star(\mu,\delta_G)$. 

\begin{proposition}\label{pro:2}
We have $K_G^\star(\mu,\delta_G) \le L_G^\star(\mu,\delta_G)$ where $L_G^\star(\mu,\delta_G)$ is the value of the optimization problem:   $\min \sum_{x,g,h}\eta(x,g,h)$ over $\eta\ge 0$ satisfying for all $\bar{g}\neq g^\star$, 
\[\max_x\min_{\bar{h}_x}(\eta(x,g^\star,h_x^\star)+\eta(x,\bar{g},\bar{h}_x))I_{\alpha_{x,\bar{g},\bar{h}_x}}(\mu(x,g^\star,h_x^\star),\mu(x,\bar{g},\bar{h}_x))\ge \kl(\delta_G,1-\delta_G).\]
\end{proposition}
One can show that $L_G^\star(\mu,\delta_G)$ scales as $GH\kl(\delta_G,1-\delta_G)$ ( since, even with one task, we need to sample all $GH$ arms to identify $g^\star$). To summarize, we have shown that the  lower bound $K^\star(\mu,\boldsymbol{\delta})$ scales at most as $H(G\kl(\delta_G,1-\delta_G)+X \kl(\delta_H,1-\delta_H))$. The latter scaling indicates the gain in terms of sample complexity one can expect when exploiting the structure of ${\cal M}$, i.e., a common optimal representation. Without exploiting this structure,  identifying the best arm for each task would result in a scaling of $GHX\kl(\delta,1-\delta)$ for $\delta=\delta_G+\delta_H$.

\paragraph{Differences with classical best-arm identification .} To better understand the lower bound in \cref{thm:full_lower_bound_offline} it is instructive to compare the with a classical MAB problem. Consider the best-arm identification problem in MAB with $K$ arms. There the set of confusing problems is $\Lambda(\mu) = \{\lambda \in[0,1]^K: a^{\star}(\lambda)\neq a^{\star}(\mu)\}$, where $a^{\star}(\mu)$ denotes the optimal arm under $\mu$, i.e., $a^{\star}(\mu)=\argmax_{a\in K} \mu(a)$. The sample complexity lower bound derived for such classical MAB problems \cite{kaufmann2016complexity} exploits the fact that the set $\Lambda(\mu)$ can be written as $\Lambda(\mu) = \bigcup_{a \neq a^{\star}(\mu)} \Lambda_a(\mu),$ where $\Lambda_a(\mu)\coloneqq\left\{\lambda \in [0,1]^K: \lambda_{a}>\lambda_{a^{\star}(\mu)} \right\}$. 

Unfortunately, this way of rewriting the set of confusing problems cannot be used in our problem setting. The reason is that the constraint in $\Lambda_a(\mu)$ does not account for the model structure, i.e., the optimal representation $g^{\star}$ needs to be the same across all the tasks (that is equivalent to imposing that $a$ is optimal for all tasks), which is also the reason the lower bound appears more complex. In Appendix A we show how to account for this kind of structure.

Because of this difference, with the model specified in \cref{eq:dominate} the  confusing parameter $\lambda$ differs from $\mu$ for more than $2$ arms (i.e., we need to consider all arms in the the set $\set{U}_{x,\bar g, \bar h_x}^\eta$, see also Lemma 1 in Appendix A), whereas in classical MAB to learn that $a$ is sub-optimal, the confusing parameter $\lambda\in \Lambda_a(\mu)$ differs from $\mu$ only for arms $a$ and $a^\star(\mu)$.
 In fact, in our model to identify that $(\bar{g},\bar{h}_x)$ is sub-optimal we need to consider an alternative model where only the average reward of the arms in the set $\set{U}_{x,\bar g, \bar h_x}^{\eta,\mu}$ changes. 
 
\noindent The set $\set{U}_{x,\bar g, \bar h_x}^{\eta,\mu}$ can be computed easily by recursively enlarging the set. We start with a set with only $(g^\star,h_x^\star)$ and 
$(\bar{g},\bar{h}_x)$. We compute the corresponding value of $m\left(x,\eta, \mathcal{U}_{x,\bar g, \bar h_x}^{\eta,\mu}\right)$, and we add to the set $\set{U}_{x,\bar g, \bar h_x}^{\eta,\mu}$ the arm $(g,h)$ with the highest mean not already in the set. We iterate until convergence. \cref{fig:computation_confusing_problem} provides an illustration of the set $\set{U}_{x,\bar g, \bar h_x}^{\eta,\mu}$.

\subsection{Algorithm}
We now present \textsc{OSRL-SC} (\cref{alg:offline_algo_full}), a $\boldsymbol{\delta}$-PAC algorithm whose expected sample complexity is asymptotically upper bounded by $K_G^\star(\mu,\delta_G) + K_{H}^\star(\mu,\delta_H)$. The algorithm proceeds in two phases: a first phase aimed at learning $g^\star$, and a second phase devoted to learning the optimal predictor for each task. At the end of the first phase, we have an estimate $\hat{g}$ of the best representation. In the second phase, for each task $x$, we use the optimal track-and-stop algorithm \cite{garivier2016optimal} to identify $\hat{h}_x$, the best predictor associated to $\hat{g}$. In the remaining part of the section, we just describe the first phase.

\paragraph{A track-and-stop algorithm to learn $g^\star$.} The lower bound describes the minimal expected numbers $\eta$ of observations of the various tasks needed to learn $g^\star$. These numbers minimize $\sum_{x,g,h}\eta(x,g,h)$ over $\eta\ge 0$ satisfying (\cref{eq:cons2}). In other words, the sampling budget should be allocated according to the following distribution: $q^\star(\mu)\in \Sigma$, solving: $\max_{q\in \Sigma} \rho(q,\mu)$, where
\begin{equation}
\rho(q,\mu)= \min_{\bar g\neq g^{\star}} \sum_{x} 
\min_{\bar h_x} \ell_\mu(q,\bar g, \bar h_x),
\end{equation}   
 and $\Sigma$ denotes the set of distributions over $\mathcal{X}\times \mathcal{G}\times \mathcal{H}$. We design an algorithm tracking this optimal allocation: it consists of (i) a sampling rule, (ii) a stopping rule, and a (iii) decision rule, described below.  

{\it (i) Sampling rule.} We adapt the D-tracking rule introduced in \cite{garivier2016optimal}. The idea is to track $q^{\star}(\mu)$, the optimal proportion of times we should sample each (task, arm) pair. One important issue is that the solution to $\max_{q\in \Sigma} \rho(q,\mu)$ is not unique (this happens for example when two tasks are identical). To circumvent this problem, we track $q_\sigma^\star(\mu)$, the unique solution of  
$(P_\sigma):\max_{q\in \Sigma} \rho(q,\mu)-\frac{1}{ 2\sigma}\| q\|_2^2$, where $\sigma>0$. When $\sigma$ is large, Berge's Maximum theorem \cite{berge1963topological} implies that $q_\sigma^\star(\mu)$ approaches the set of solutions of $\max_{q\in \Sigma} \rho(q,\mu)$, and that the value $C_\sigma(\mu)$ of $(P_\sigma)$ converges to $K_G^\star(\mu,\delta_G)/\skl(\delta_G,1-\delta_G)$. In what follows, we let $K_{G,\sigma}^\star(\mu,\delta_G)\coloneqq C_\sigma(\mu)\skl(\delta_G,1-\delta_G)$.

Our D-tracking rule targets $q_\sigma^\star(\hat{\mu}_t)$, which is the unique maximizer of $\max_{q\in \Sigma}\rho(q,\hat \mu_t)-\frac{1}{2\sigma}\|q\|_2^2$, where $\rho(q,\hat{\mu}_t)$ for any $\hat\mu_t\in \set M$ is defined as
\begin{equation}
\rho(q,\hat\mu_t)= \min_{\bar g\neq g_t^{\star}} \sum_{x}  \min_{\bar h_x} \ell_{\hat \mu_t}(q,\bar g, \bar h_x).
\end{equation}
 More precisely, if the set of under-sampled tasks and arms $U_t = \{(x,g,h)\in \mathcal{X}\times \mathcal{G}\times \mathcal{H}:  N_t(x,g,h) <\sqrt{t} -GHX/2 \}$ is not empty, or when $\hat{\mu}_t\notin \set M$, we select the least sampled (task, arm) pair. 
Otherwise, we track $q_\sigma^\star(\hat{\mu}_t)$, and select $ \argmax_{(x,g,h)} tq_{\sigma}^\star(x,g,h;\hat{\mu}_t)-N_t(x,g,h)$.

{\it (ii) Stopping rule.} We use Chernoff's stopping rule, which is formulated using a Generalized Likelihood Ratio Test \cite{chernoff1959sequential}. The derivation of this stooping rule is detailed in appendix B. The stopping condition is $\max_{\tilde{g}} \Psi_t(\tilde{g}) >\beta_t(\delta_G)$, where the exploration threshold $\beta_t(\delta_G)$ will be chosen appropriately, and where
\[
\Psi_t(\tilde{g})=\min_{\bar g\neq \tilde{g}} \sum_{x} \min_{\bar{h}_x} \ell_{\hat \mu_t}(N_t,\bar g, \bar h_x).  
\]

\begin{algorithm}[!t]
	\caption{\textsc{OSRL-SC} $(\delta_G,\delta_H,\sigma)$}
	\label{alg:offline_algo_full}
	\begin{algorithmic}[1]
	\STATE {\bf Initialization.} \STATE $N_1(x,g,h),\hat{\mu}_1(x,g,h)\gets 0$, $\forall (x,g,h) \in \mathcal{X}\times\mathcal{G}\times \mathcal{H}$ 
	\STATE $t\gets 1$
	\STATE {\bf Phase 1. Learning $g^\star$} 
	\WHILE{$\max_{g\in G}\Psi_t(g) \leq  \beta_t(\delta_G)$}
			\IF{$U_t = \emptyset$ and $\hat{\mu}_t\in \set M$}
			\small
				\STATE{$(x(t),g(t),h(t))\gets \argmax_{(x,g,h)} tq_{\sigma}^{\star}(x,g,h;\hat{\mu}_t)-N_t(x,g,h)$} 
				\normalsize
			\ELSE
			\STATE{$(x(t),g(t),h(t))\gets \argmin_{(x,g,h)} N_t(x,g,h)$} 
			\ENDIF
			\STATE{Select $(x(t),g(t),h(t))$, observe the corresponding reward and update statistics; $t\leftarrow t+1$}
		\ENDWHILE
		\RETURN $\hat{g} = \argmax_g \hat{\mu}_{\tau_G}(g)$
	\STATE {\bf Phase 2. Learning $(h_1^\star,\ldots,h_X^\star)$} 
	\STATE { For all task $x$, $\hat{h}_x\leftarrow$ [track-and-stop \cite{garivier2016optimal} with arms $(\hat{g},h)_{h\in {\cal H}}$ and confidence $\delta_H$]}
		\RETURN $(\hat{g},\hat h_1,\ldots,\hat h_X)$
	\end{algorithmic}
\end{algorithm}
{\it (iii) Decision rule.} The first phase ends at time $\tau_G$, and $\hat{g}$ is chosen as the best empirical representation: $\hat{g}=\arg\max_g \hat{\mu}_t(g)$.  
 \begin{table*}[t]
 \setlength\arrayrulewidth{1pt}  
 \resizebox{\textwidth}{!}{
	\begin{tabular}{lll|cc|ccc}\toprule
		&&&\textbf{Average} & \textbf{Confidence interval $97.5\%$} & \textbf{Min} & \textbf{Max} & \textbf{Std} \\\toprule 
		\rowcolor[gray]{.95}$\delta=0.1$ & \textsc{OSRL-SC}& \textit{Total}& ${21278.80}$ & $\pm {430.37}$ & $5254.0$ & $46876.0$ & ${6423.03}$\\
		\rowcolor[gray]{.95}& &  \textbf{\textit{Phase 1}} & $\mathbf{3578.38}$ & $\pm \mathbf{43.31}$ & $\mathbf{2163.0}$ & $\mathbf{7014.0}$ & $\mathbf{646.46}$\\
		\rowcolor[gray]{.95}& & \textit{Phase 2} & ${17700.42}$ & $\pm {428.39}$ & $1554.0$ & ${43270.0}$ & ${6393.44}$\\
		     & \textsc{TaS} & &$26456.83$ & $\pm 510.60$ &  ${4544.0}$ & $54566.0$ &$7620.35$\\\midrule
		\rowcolor[gray]{.95}$\delta=0.05$  & \textsc{OSRL-SC} &\textit{Total}& ${22671.50}$ & $\pm{420.40}$ & $6068.0$ & $48184.0$ & ${6274.13}$ \\
		\rowcolor[gray]{.95} & &  \textbf{\textit{Phase 1}} & $\mathbf{3651.99}$ & $\pm \mathbf{41.86}$ & $\mathbf{2358.0}$ & $\mathbf{6245.0}$ & $\mathbf{624.72}$\\
				\rowcolor[gray]{.95}& & \textit{Phase 2} & ${19019.51}$ & $\pm {417.81}$ & $2207.0$ & ${45298.0}$ & ${6235.60}$\\
		& \textsc{TaS} && $27735.38$ & $\pm 534.35$ & ${7675.0}$& $58227.0$ & $7974.87$\\\midrule
	
		\rowcolor[gray]{.95}$\delta=0.01$ & \textsc{OSRL-SC} &\textit{Total}& ${25765.90}$ & $\pm{436.13}$ & $8951.0$ & $55809.0$ & ${6508.94}$ \\
		\rowcolor[gray]{.95} & &  
		\textbf{\textit{Phase 1}} &$\mathbf{3829.56}$ & $\pm\mathbf{45.93}$ & $\mathbf{2358.0}$ & $\mathbf{7354.0}$ & $\mathbf{685.44}$ \\
				\rowcolor[gray]{.95}& & \textit{Phase 2} & ${21936.34}$ & $\pm {434.09}$ & $5398.0$ & ${52002.0}$ & ${6478.57}$\\
		    & \textsc{TaS} && $30970.94$ & $\pm 536.99$ & ${9538.0}$ & $70319.0$ & $8014.26$
		\\\bottomrule
	\end{tabular}}
	\caption{\textsc{OSRL-SC} vs \textsc{TaS}: Sample complexity results, over 1120 runs.}\label{Tab1}
\end{table*}


\subsection{PAC and sample complexity analysis}
We now present the sample complexity upper bound for \cref{alg:offline_algo_full}. First, we outline the stopping rule used by the algorithm. Following \cite{kaufmann2018mixture}, we first need to define $\phi : \mathbb{R}_+\to\mathbb{R}_+$ as $\phi(x) = 2\tilde{p}_{3/2}\left(\frac{p^{-1}(1+x)+\ln(2\zeta(2))}{2}\right)$, where $\zeta(s) = \sum_{n\geq 1} n^{-s}$, $p(u)=u-\ln(u)$ for $u\geq 1$ and for any $z\in [1,e]$ and $x\geq 0$:
\[
\tilde{p}_z(x) = \begin{cases}
e^{1/p^{-1}(x)}p^{-1}(x)\quad \hbox{ if } x \geq p^{-1}(1/\ln z),\\
z(x-\ln\ln z)\quad \hbox{otherwise.}
\end{cases}
\]
Then, the following theorem states that with a carefully chosen exploration threshold, the first phase of \textsc{OSRL-SC} returns the optimal representation w.p. greater than $1-\delta_G$.

\begin{theorem}\label{theorem:deltapac}
Let $\delta_G\in(0,1)$: for any sampling rule, Chernoff's stopping rule with threshold $\beta_t(\delta_G)=\beta_1(t) + \beta_2(\delta_G)$, where $\beta_1(t) = 3\sum_{x,g,h}\ln(1+\ln(N_t(x,g,h)))$, and  $\beta_2(\delta_G) = GHX \phi(\ln((G-1)/\delta_G)/XGH)$, ensures that for all $ \mu \in\set{M}$, $\mathbb{P}_\mu(\tau_{G} <\infty, \hat{g} \neq g^{\star}) \leq \delta_G$.
\end{theorem}

From \cite{garivier2016optimal}, the second phase of \textsc{OSRL-SC} also returns the optimal predictors  for each task w.p. greater than  $1-\delta_H$. Finally, in the next theorem, we show that \textsc{OSRL-SC} stops in finite time a.s., and that its expected sample complexity approaches  $K_G^\star(\mu,\delta_G) + K_{H}^\star(\mu,\delta_H)$ for sufficiently small values of the risks $\delta_G, \delta_H$, and sufficiently large $\sigma$. 

\begin{theorem}\label{theorem:exp_sc_bound}
If the exploration threshold of the first phase of \textsc{OSRL-SC} is chosen as in Theorem \cref{theorem:deltapac}, then we have: $\mathbb{P}_\mu[\tau_G<\infty]=1$ and $\mathbb{P}_\mu[\tau_H<\infty]=1$ (where $\tau_H$ is the time at which the second phase ends). In addition, the sampling complexity of \textsc{OSRL-SC} satisfies: $\limsup_{\delta_G,\delta_H\to 0} \frac{\mathbb{E}_\mu[\tau]}{K_{G,\sigma}^\star(\mu,\delta_G)+K_H^\star(\mu,\delta_H)} \leq 1$,
where $K_{G,\sigma}^\star(\mu,\delta_G)=C_\sigma(\mu)\skl(\delta_G,1-\delta_G)$, with $C_\sigma(\mu)\coloneqq \left( \max_{q\in \Sigma} \rho(q,\mu) -\frac{1}{2\sigma}\|q\|_2^2\right)^{-1}$.
\end{theorem}

\begin{corollary}Additionally, due to Berge's theorem, since $\rho(q,\mu)- \frac{1}{2\sigma}\|q\|_2^2$ is continuous in $q$ for each $(\sigma, \mu)$, we have: $\lim_{\sigma \to \infty} K_{G,\sigma}^\star(\mu,\delta_G)=K_G^\star(\mu,\delta_G)$.
\end{corollary}
\section{Numerical results}~\label{app:exp}
\noindent We analyze the performance of \textsc{OSRL-SC}, and compare it directly with \textsc{Track and Stop} (\textsc{TaS})\cite{garivier2016optimal}. We are interested in answering the following question: {is it easier to learn the best representation by just focusing on one task, or should we consider multiple tasks at the same time?}
%

\paragraph{Simulation setup.}
We consider $2$ tasks ($x_1,x_2$), $3$ representations ($g_1,g_2,g_3$) and $2$ predictors ($h_1,h_2$). This setting is rather simple, although not trivial. Note that as the number of (task, arm) pairs decreases, we expect the gap between the two algorithms to decrease and thus favour \textsc{TaS}. Hence considering examples with small numbers of tasks are informative about OSRL abilities to factor information across tasks. The average rewards are
\renewcommand{\kbldelim}{(}
\renewcommand{\kbrdelim}{)}
\[
 \underbrace{\kbordermatrix{
    & h_1 & h_2 \\
    g_1 & 0.5 & 0.45 \\
    g_2 & 0.35 & 0.33 \\
    g_3 & 0.1 & 0.05  
  }}_{\text{Average rewards for task } x_1},\qquad 
  \underbrace{\kbordermatrix{
      & h_1 & h_2 \\
      g_1 & 0.5 & 0.45 \\
      g_2 & 0.1 & 0.05 \\
      g_3 & 0.35 & 0.33  
    }}_{\text{Average rewards for task } x_2}
\]
We set up tasks $x_1$ and $x_2$ so that they are very similar: actually, the only difference is that the 2nd and 3rd row of the above matrices are swapped. Therefore it should not matter which task \textsc{TaS} picks, but, on the other hand, \textsc{OSRL-SC} should benefit from this small difference. Hence, for each simulation of \textsc{TaS}, we picked one task uniformly at random. We average results over 1120 runs.

\paragraph{Algorithms.} We test \textsc{TaS} and \textsc{OSRL-SC} with various risks $\delta \in \{0.01, 0.05, 0.1\}$ (with $\delta=\delta_G=\delta_H$ for \textsc{OSRL-SC}). For \textsc{TaS}, we use the following threshold $\beta_t(\delta_G) = \log\left(\frac{2t(GH-1)}{\delta_g}\right)$. We tried the same threshold as in  \textsc{OSRL-SC}, but it yielded worse results.
For \textsc{OSRL-SC}, we set $\sigma=10^5$. For the example considered, one can see that $\argmax_{q\in\Sigma}\rho(q,\mu)$ has a unique maximizer. Therefore, $\sigma$ will not influence the value of the lower bound if $\hat\mu_t$ is approximately close to $\mu$, in norm. However, when $\hat\mu_t$ is visibly different from $\mu$, then the value of $C_\sigma(\hat \mu_t)$ may be affected by the value of $\sigma$. We have not thoroughly explored different values of $\sigma$, but we may suggest that a value of $\sigma> 10^3$ is a safe choice.

We compute $q_\sigma^{\star}(\hat \mu_t)$ every $12$ rounds (which is equal to $GHX$) to reduce the computational time (this is better motivated in Appendix B). Despite that, one needs to keep in mind that tracking a sub-optimal, or wrong, reference vector $q_\sigma^\star$ may sensibly affect the sample complexity. We can also motivate this period update by the fact that $q_\sigma^{\star}(\hat \mu_t)$ in a small time interval does not change much, as shown numerically in what follows. 

To compute $q_\sigma^\star(\hat \mu_t)$, in round $t+1$,  we use as initial condition a convex combination of the previous solution and a uniform point in the probability simplex (with a factor $0.5$). This is done to speed-up the algorithm (for more details refer to Appendix B).

\paragraph{Comparison of \textsc{OSRL-SC} and \textsc{TaS}.}

In Table \ref{Tab1}, we report the sample complexity of the two algorithms. In bold, we highlighted results for the first phase of \textsc{OSRL-SC}. Even if the number of representations is higher than the number of predictors, somewhat surprisingly, the first phase \textsc{OSRL-SC} looks  on average very stable, with a small confidence interval (when compared to Phase 2 or \textsc{TaS}). It is worth observing that with the smallest number possible of tasks ($X=2$), \textsc{OSRL-SC} manages to reduce the required number of rounds to identify the optimal representation, and the predictors, compared to  \textsc{TaS}. Furthermore, the first phase of \textsc{OSRL-SC} appears  stable also when $\delta$ decreases. Between $\delta=0.1$ and $\delta=0.05$ there is a relative increase of average sample complexity of roughly $2\%$ for \textsc{OSRL-SC}; between $\delta=0.05$ and $\delta=0.01$ we have that the average sample complexity for \textsc{OSRL-SC} has a relative increase of roughly $5\%$. Overall, these results indicate that \textsc{OSRL-SC} is able to re-factor information efficiently.

\paragraph{Analysis of \textsc{OSRL-SC}: First phase}
\begin{figure}[h!]
	\begin{minipage}{1\textwidth}
		\centering
		\subcaptionbox{}{\includegraphics[width=.49\linewidth]{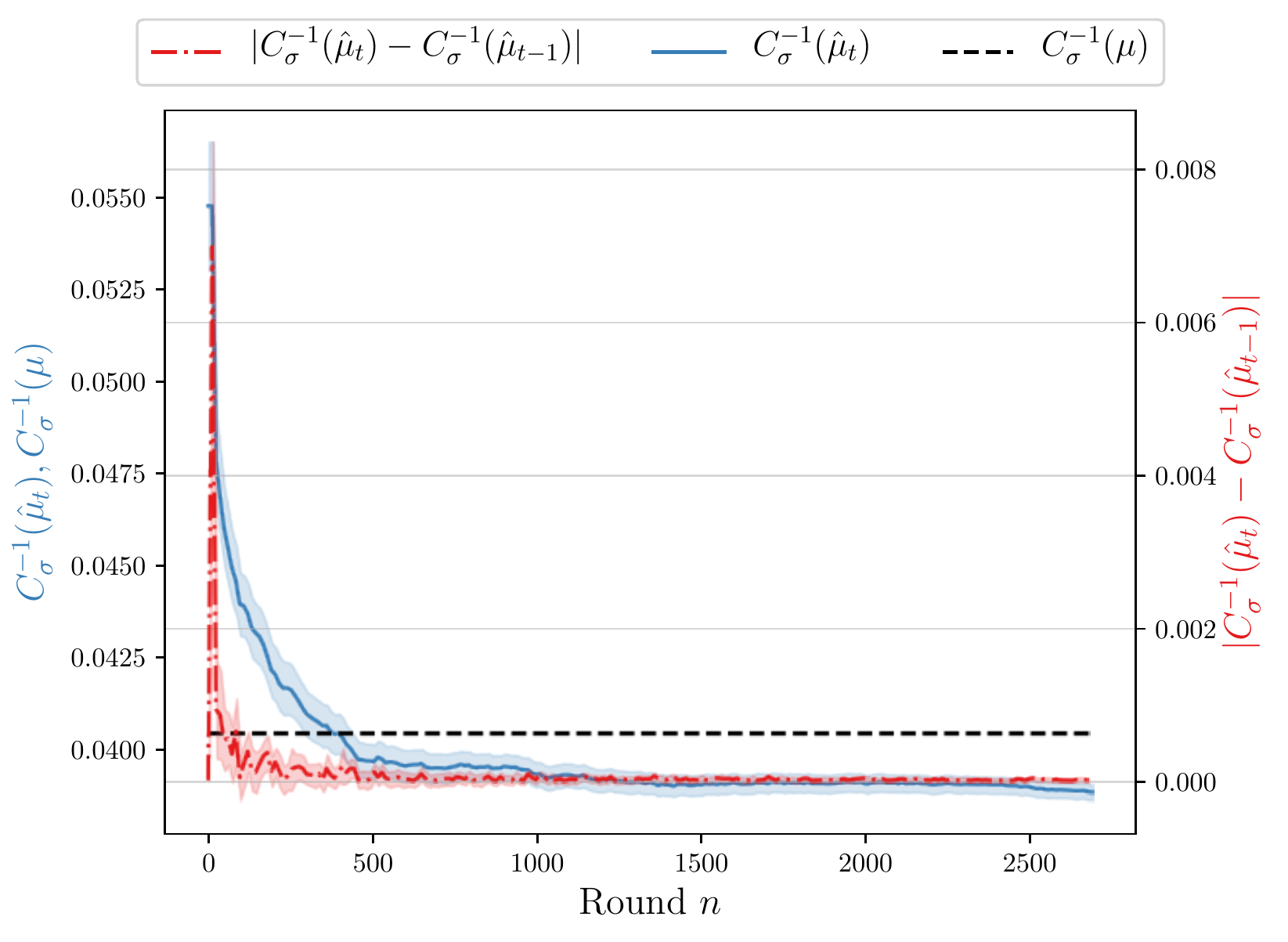}}
		\subcaptionbox{}{\includegraphics[width=.49\linewidth]{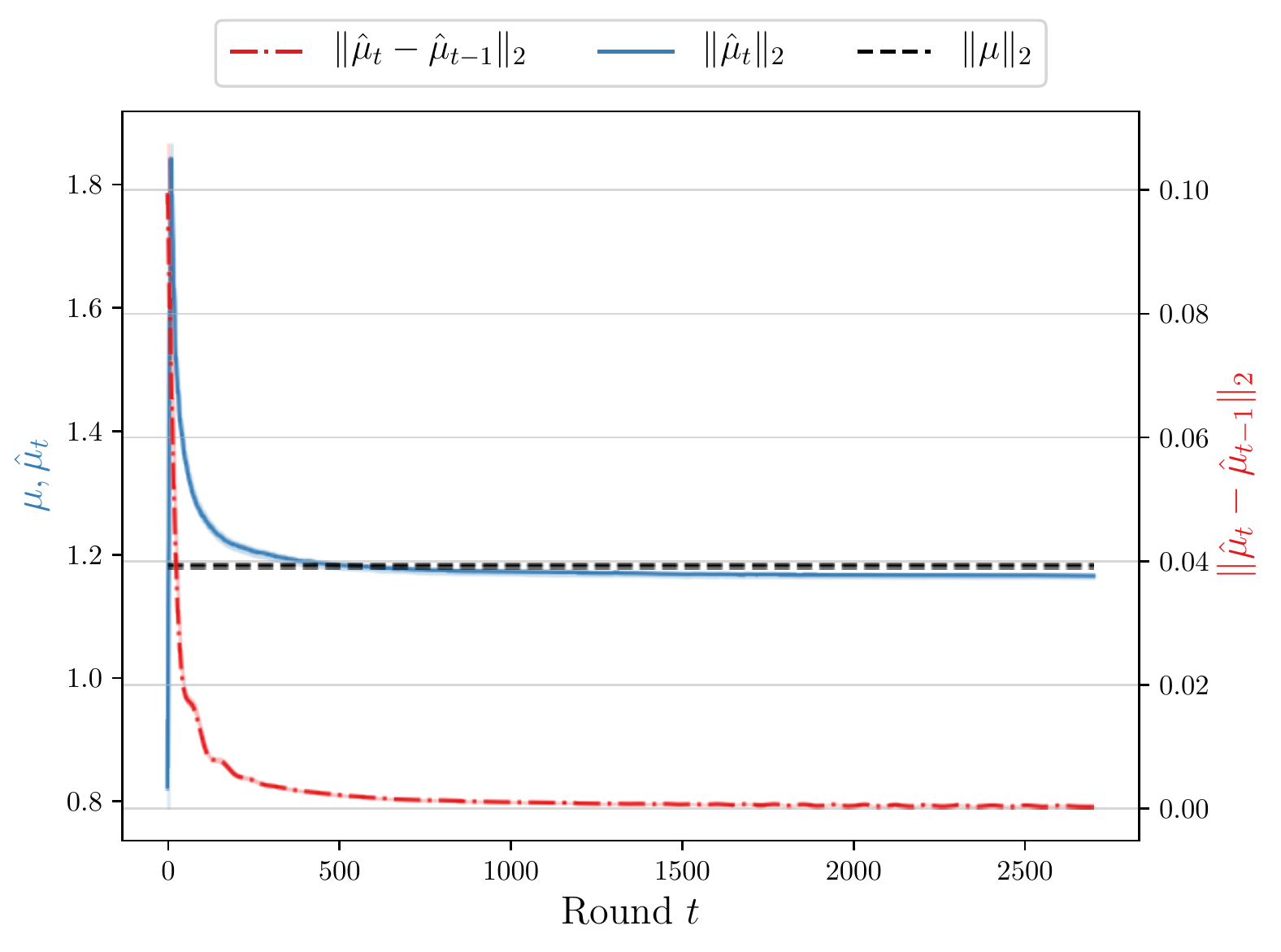}}
		\caption{Analysis of $C_\sigma^{-1}(\hat \mu_t)$ and $\hat \mu_t$ under \textsc{OSRL-SC}. (a) Average dynamics of $C_\sigma^{-1}(\hat{\mu}_t)$, (b) Dynamics of $\hat \mu_t$. $\|\hat{\mu}_t-\hat{\mu}_{t-1}\|_2$ is normalized by $\sqrt{GHX}$ to show the average change of each component, and low-pass filtered using a $8$-th order Butterworth filter with critical frequency $\omega_0=0.025$.  The shadowed areas represent $97.5\%$ confidence interval.  } 
		\label{fig:C_sigma_analysis}
	\end{minipage}
\end{figure}
\begin{table}[t]
 \resizebox{\linewidth}{!}{
\begin{tabular}{lcccccc}
\toprule
\textbf{Task $x_1$}   & $(g_1,h_1)$          & $(g_1,h_2)$          & $(g_2,h_1)$          & $(g_2,h_2)$          & $(g_3,h_1)$          & $(g_3,h_2)$          \\ \midrule
 $N_{\tau_G}(x_1)/\tau_G$          & 0.14                 & 0.05                 & 0.02                 & 0.02                 & 0.12                 & 0.15                 \\
$q_\sigma^\star(x_1;\mu)$ & 0.18                 & $5\cdot 10^{-3}$     & $6\cdot10^{-4}$      & $7\cdot10^{-4}$      & 0.13                 & 0.18                 \\ \midrule
                      & \multicolumn{1}{l}{} & \multicolumn{1}{l}{} & \multicolumn{1}{l}{} & \multicolumn{1}{l}{} & \multicolumn{1}{l}{} & \multicolumn{1}{l}{} \\ \midrule
\textbf{Task $x_2$}   & $(g_1,h_1)$          & $(g_1,h_2)$          & $(g_2,h_1)$          & $(g_2,h_2)$          & $(g_3,h_1)$          & $(g_3,h_2)$          \\ \midrule
$N_{\tau_G}(x_2)/\tau_G$          & 0.14                 & 0.05                 & 0.12                 & 0.16                 & 0.02                 & 0.02                 \\
$q_\sigma^\star(x_2;\mu)$ & 0.18                 & $5\cdot 10^{-3}$     & 0.13                 & 0.18                 & $6\cdot10^{-4}$      & $7\cdot10^{-4}$      \\ \bottomrule
\end{tabular}}
\caption{Analysis of \textsc{OSRL-SC}. Comparison of the optimal allocation vector $q_\sigma^\star(\mu)$ and the average proportion of arm pulls $N_{\tau_G}/\tau_G$ at the stopping time.}\label{Tab2}
\end{table}

To analyze the convergence of \textsc{OSRL-SC}, we focus on its first phase, specifically on the following quantities: $\hat \mu_t, q_\sigma^{\star}(\hat \mu_t)$ and $C_\sigma^{-1}(\hat{\mu}_t)$. We will use the right y-axis of each plot to display the difference between the value in round $t$ and round $t-1$ of the quantities considered.

Figure \ref{fig:C_sigma_analysis}(b) shows how $\|\hat{\mu}_t-\hat{\mu}_{t-1}\|_2$ (normalized by  $\sqrt{XGH}$) and $\|\hat{\mu}_t\|_2$ evolve over time. We clearly see that $\|\hat{\mu}_t\|_2$ quickly converges to some fixed value. This convergence appears in all the plots. Figure  \ref{fig:C_sigma_analysis}(a) shows the value of $C_\sigma^{-1}(\hat \mu_t)$, the true value $C_\sigma^{-1}(\mu)$, and the relative change of $C_\sigma^{-1}(\hat \mu_t)$ between two consecutive steps. We observe that the convergence rate of $\hat{\mu}_t$ dictates also the convergence of $C_\sigma^{-1}(\hat{\mu}_t)$. This suggests that we do not need to solve the lower bound optimization problem too often to update the target allocation, which helps  reduce the computational complexity.

Figures \ref{fig:q_sigma_analysis}(a) and (b) show 2  curves each: the left plot shows the $2$-norm of $q_\sigma^{\star}(\hat \mu_t)$ and  $q_\sigma^{\star}(\hat \mu_t)-q_\sigma^{\star}(\hat \mu_{t-1})$ (the latter normalized by $\sqrt{GHX}$), whilst the right plot shows the same signals computed using the $L^\infty$-norm. In Figure \ref{fig:q_sigma_analysis}(b), notice that the average absolute change in each component of the reference vector is very small, below $3\%$ after few dozens of steps. Furthermore, we can see that this quantity has a convergence rate that is directly dictated by the convergence of  $\hat{\mu}_t$ (even if its convergence rate is smaller). In Figure \ref{fig:q_sigma_analysis}(a), observe that the relative difference between $q_\sigma^{\star}(\hat \mu_t)$ and $q_\sigma^{\star}(\mu)$ around $t=2500$ is upper bounded by roughly $1/9$.

%
%

\begin{figure}[!t]
	\begin{minipage}{1\textwidth}
		\centering
		\subcaptionbox{}{\includegraphics[width=.49\linewidth]{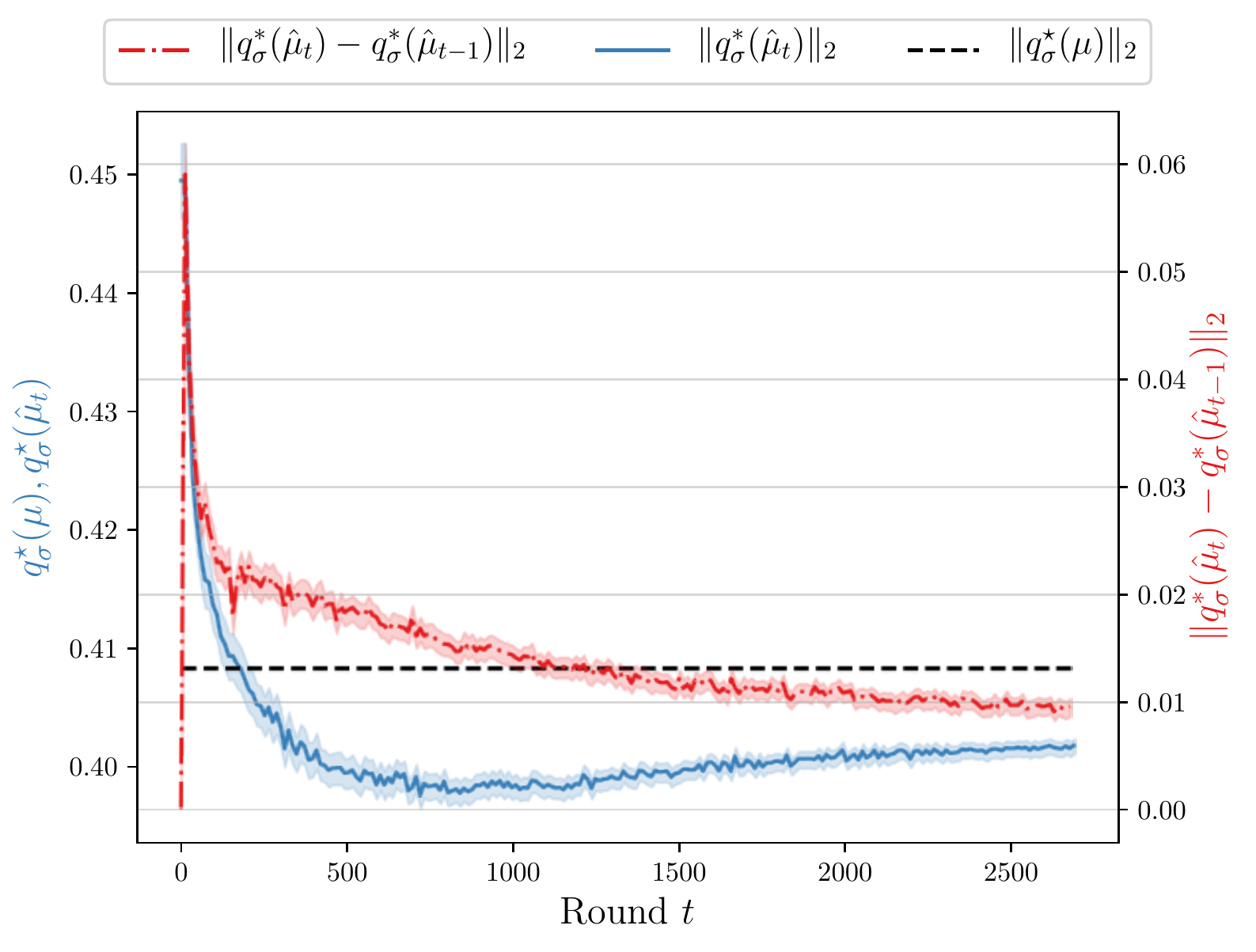}}
		\subcaptionbox{}{\includegraphics[width=.49\linewidth]{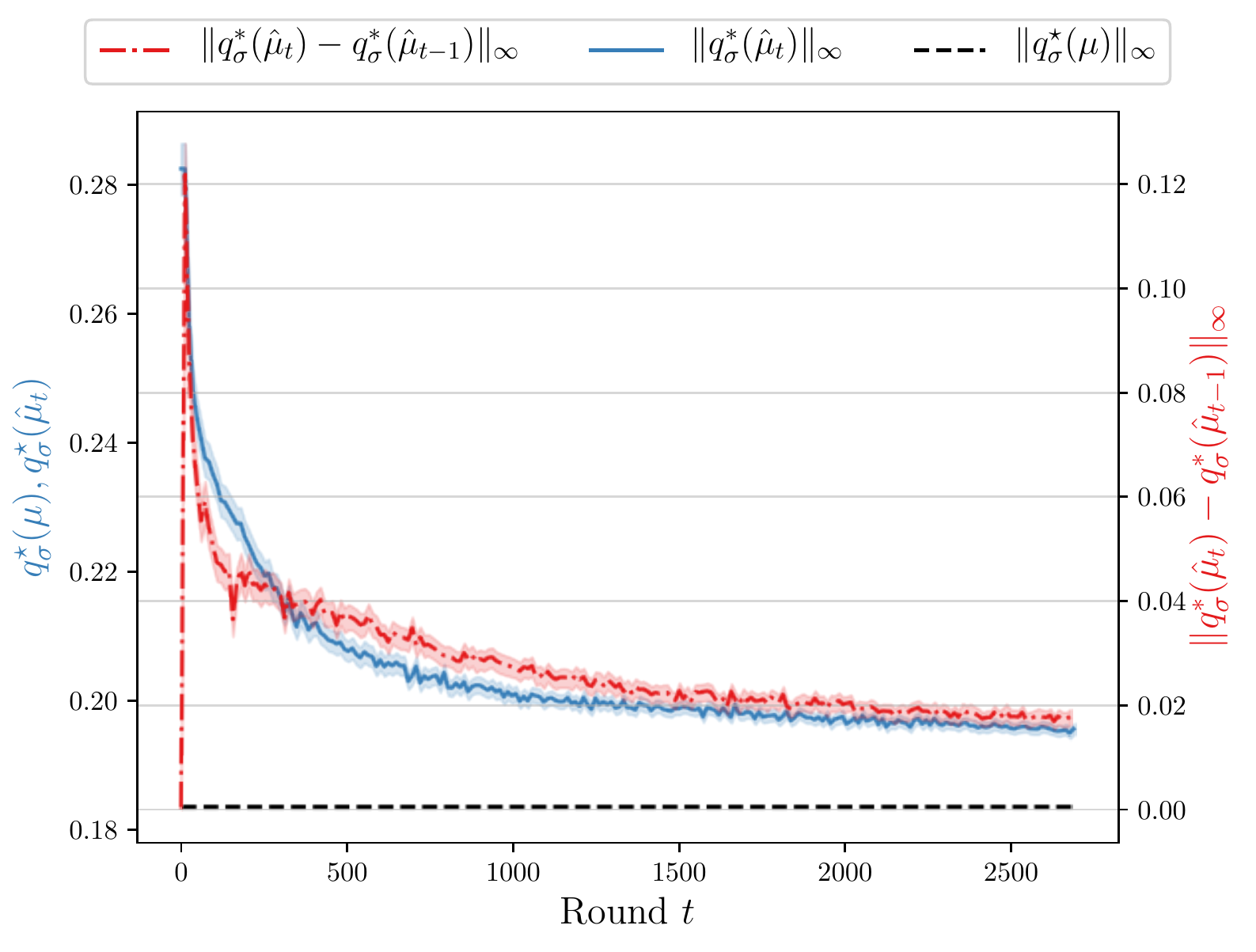}}
	\end{minipage}
	\caption{Analysis of $q_\sigma^{\star}(\hat \mu_t)$ under \textsc{OSRL-SC}. (a) Results using the $2$-norm; (b) Results using the $L^\infty$-norm.  $\|q_\sigma^{\star}(\hat \mu_t)-q_\sigma^{\star}(\hat \mu_{t-1})\|_2$ is normalized by $\sqrt{GHX}$ to show the average change for each component.
	The shadowed areas represent $97.5\%$ confidence interval.}
	\label{fig:q_sigma_analysis}
\end{figure}

Finally, and importantly, in Table \ref{Tab2}, we show the average proportion of arm pulls under \textsc{OSRL-SC} at the stopping time $\tau_G$ compared to the optimal allocation vector $q_\sigma^\star(\mu)$. It turns out that \textsc{OSRL-SC} follows accurately the optimal allocation. The algorithm picks the most informative arms in each task, i.e., it adapts to the task. From this table, we can answer our initial question: to learn $g^\star$ as fast as possible, we need to use all tasks. Task 1 is used to learn that $g_3$ is sub-optimal, and Task 2 is used to learn that $g_2$ is sub-optimal. This is precisely what \textsc{OSRL-SC} is doing.

\section{Conclusion}\label{sec7}
In this work, we analyzed knowledge transfer in stochastic multi-task bandit problems with finite arms, using the framework of  best-arm identification with fixed confidence. We proposed \textsc{OSRL-SC}, an algorithm that  transfers knowledge across tasks while approaching the sample complexity lower bound. 
We believe that this paper constitutes a sound starting point to study the transfer of knowledge in more general online multi-task learning problems.
The limitation of this work is that we only consider models with a finite number of tasks and arms, which could limit their application in real life. Furthermore, our algorithm converges to an upper bound of the lower bound. Whether it is possible to approach the sample complexity lower bound is an interesting area of research. Lastly, we think it would be interesting to study different structural assumptions (e.g. linearity) that tie reward functions across tasks together, or extend this work to multi-task reinforcement learning in MDPs.

%
%
%

%
%

\addcontentsline{toc}{section}{References}
\bibliographystyle{plainnat}
\bibliography{ref}
\appendix

\section{Lower bound}

In this section, we provide the proofs of Theorem \ref{thm:full_lower_bound_offline}, and of Propositions \ref{pro:1}-\ref{pro:2}. 

\subsection{Proof of the sample complexity lower bound of Theorem \ref{thm:full_lower_bound_offline}}\label{sub:c2}

The proof relies on classical change-fo-measure arguments, as those used in the classical MAB \cite{kaufmann2016complexity}. 

\begin{proof}
To simplify the notation, let $\tau=\tau_{\boldsymbol{\delta}}$, and further let $\eta (x,g,h) = \mathbb{E}_\mu[N_\tau(x,g,h)]$ at the stopping time $\tau$ (notice that $\mathbb{E}_\mu[\tau] = \sum_{x,g,h} \eta (x,g,h)$). For any model $\mu \in \set M$ we will denote the optimal representation of $\mu$ by $g^\star(\mu)$ and the optimal set of predictors (associated to $g^\star(\mu)$) by $\boldsymbol{h}^\star(\mu)=(h_1^\star, \dots,h_X^\star )(\mu)$. Whenever possible, we will write $g^\star=g^\star(\mu)$ (similarly for $h_x^\star=h_x^\star (\mu),\forall x\in \mathcal{X}$).\\

Define the set of \textit{confusing problems} \[\Lambda (\mu)\coloneqq \left\{\lambda\in {\cal M}: (g^\star,h_1^\star,\ldots,h_X^\star)(\lambda)\neq (g^\star,h_1^\star,\ldots,h_X^\star)(\mu)\right\}.\] 
As in the analysis of the regret lower bound, we will split the analysis by considering two subsets of $\Lambda$,  defined as follows:
\begin{equation*}
	 {\Lambda}_1^\mu\coloneqq \left\{\lambda \in  \Lambda(\mu): g^\star(\lambda)=g^\star(\mu) \right\},\quad {\Lambda}_2^\mu \coloneqq \left\{\lambda \in  \Lambda(\mu): g^\star(\lambda)\neq g^\star(\mu) \right\}.
\end{equation*}
We will now focus on the first subset $\Lambda_1^\mu$, from which we will derive the first constraint (\ref{eq:cons1}) of Theorem \ref{thm:full_lower_bound_offline}. Then, we will focus on the second subset $\Lambda_2^\mu$, from which derives the second constraint (\ref{eq:cons2}).

\underline{\textbf{First constraint (\ref{eq:cons1}).}}\\
We restrict our attention to ${\Lambda}_1^\mu$. Define the set of confusing problems for task $x \in \set X$ as:  ${\Lambda}_1^\mu(x) \coloneqq \{\lambda \in {\Lambda}_1^\mu: h_y^\star(\lambda)=h_y^\star(\mu), \forall y\in \set X\setminus \{x\}\}$, which is a handy definition that we will use in few steps.
 
 Now, consider a $(\delta_G,\delta_H)$-PAC algorithm, and for a specific task $x\in \set X$ define the event $\set{E} = \{\hat{h}_{x} \neq h_{x}^\star, \hat g=g^\star\}$, where $\hat h_x$ and $\hat g$ denote respectively the estimated predictor for task $x$ and the estimated optimal representation at the stopping time $\tau$. Let then $\lambda\in{\Lambda}_1^\mu(x)$, be an alternative bandit model: the expected log-likelihood ratio $L_{\tau}$ at the stopping time $\tau$ of the observations under the two models $\mu$ and $\lambda$ is given by
  \[
  \mathbb{E}_\mu[L_\tau] = \sum_{(y,g,h)\in \set X\times\set G\times \set H}\eta (y,g,h)\skl_{\mu|\lambda}(y,g,h),
  \]
 
 and in view of the \textit{transportation Lemma} 1 in \cite{kaufmann2016complexity} and the definition of $(\delta_G,\delta_H)$-PAC algorithm, we can lower bound the previous quantity at the stopping time $\tau$:
 \[
 \mathbb{E}_\mu[L_\tau] = \sum_{y,g,h}\eta(y,g,h)\skl_{\mu|\lambda}(y,g,h)\geq \skl(\mathbb{P}_\mu(\set{E}), \mathbb{P}_\lambda(\set{E}))=\skl(\delta_H, 1-\delta_H). 
 \]
Then, we can get a tight lower bound by considering the worst possible model $\lambda$. To do so, first introduce the set
 ${\Lambda}_1^\mu(x,h) = \{\lambda \in {\Lambda}_1^\mu(x):  \lambda (x,g^\star,h)> \lambda (x, g^\star,h_x^\star) \}$ which is the set of confusing problems where the predictor $h_x^\star$ is not optimal in task $x$. Observe that one can write $\Lambda_1^\mu(x)= \cup_{h\neq h^\star_x}\Lambda_1^\mu(x,h)$. This rewriting allows us  to derive the first constraint as follows:
\begin{align*}
\skl(\delta_H, 1-\delta_H) &\leq \inf_{\lambda \in {\Lambda}_1^\mu(x)}\sum_{y,g,h} \eta (y,g,h)\skl_{\mu|\lambda}(y,g,h),\\
				&\stackrel{(a)}{=} \inf_{\lambda \in {\Lambda}_1^\mu(x)}\sum_{h} \eta (x,g^\star,h)\skl_{\mu|\lambda}(x,g^\star,h),\\
				&= \min_{h\neq h_x^\star}\inf_{\lambda \in {\Lambda}_1^\mu(x,h)}\left[\eta (x,g^\star,h_x^\star)\skl_{\mu|\lambda}(x,g^\star,h_x^\star)+\eta (x,g^\star,h) \skl_{\mu|\lambda}(x,g^\star,h)\right],\\
				&\stackrel{(b)}{=}\min_{h\neq h_x^\star}  (\eta (x,g^\star,h_x^\star)+\eta (x,g^\star,h)) I_{\alpha_{x,g^\star,h}}(\mu(x,g^\star,h_x^\star), \mu(x,g^\star,h)).
\end{align*}
where (a) follows from the fact that in $\lambda$ we are changing the predictor of only one task $x$; (b) follows by solving the infimum problem as in Lemma 3 of \cite{garivier2016optimal}, and from the definition of generalized Shannon divergence $I_\alpha(x,y)$ with $\alpha_{x,g,h}$ defined as
\[
\alpha_{x,g,h} =  \frac{\eta (x,g^\star,h_x^\star)}{\eta (x,g^\star,h_x^\star)+\eta (x,g,h)}.
\]

\underline{\textbf{Second constraint (\ref{eq:cons2}).}}\\
 Similarly to the previous case consider a $(\delta_G,\delta_H)-$PAC algorithm and define the event $\set{E} = \{ \hat g\neq g^\star\}$: then, we can apply the transportation Lemma 1 in \cite{kaufmann2016complexity} at the stopping time $\tau$ to obtain: for every $\lambda\in {\Lambda}_2^\mu$,
 \begin{equation}\label{eq:c12}
\skl(\delta_G,1-\delta_G) \leq 	\sum_{x,g,h} \eta (x,g,h)\skl_{\mu|\lambda}(x,g,h) .
 \end{equation}
We will consider subsets of ${\Lambda}_2^\mu$ defined as follows: for every $\bar g \in \set G$ such that $ \bar g\neq g^\star$ we define $ {\Lambda}_2^\mu (\bar{g})\coloneqq \left\{\lambda \in  {\Lambda}_2^\mu: g^\star(\lambda)=\bar g \right\}$. Similarly, we also define
\[
 \quad {\Lambda}_2^\mu (\bar{g},\bar{\boldsymbol{h}})\coloneqq \left\{\lambda \in  {\Lambda}_2^\mu(\bar g): (g^\star,h_1^\star,\ldots,h_X^\star)(\lambda)=(\bar{g},\bar{\boldsymbol{h}}) \right\},\forall (\bar{g},\bar{\boldsymbol{h}})\in  (\mathcal{G}\setminus\{g^\star\})\times \mathcal{H}^\mathcal{X} ,
	\]
	where $(\bar{g},\bar{\boldsymbol{h}})=(\bar{g},\bar{h}_1,\ldots,\bar{h}_X)$. Furthermore, notice that ${\Lambda}_2^\mu=\cup_{\bar g \neq g^\star}\Lambda_2^\mu(g)=\cup_{\bar{g}\neq g^\star} \cup_{\bar{\boldsymbol{h}}\in \mathcal{H}^\mathcal{X}} {\Lambda}_2^\mu(\bar{g},\bar{\boldsymbol{h}})$. Hence, the bound in (\ref{eq:c12}) can be reformulated in the following manner:
\begin{align*}
\skl(\delta_G, 1-\delta_G) &\leq \inf_{\lambda \in \Lambda_2^\mu} \sum_{x,g,h} \eta (x,g,h)\skl_{\mu|\lambda}(x,g,h),\\
&=\min_{\bar g \neq g^\star} \inf_{\lambda \in \Lambda_2^\mu(\bar g)} \sum_{x,g,h} \eta (x,g,h)\skl_{\mu|\lambda}(x,g,h),\\
&= \min\limits_{\substack{\bar{g}\neq g^\star\\  \bar{\boldsymbol{h}}\in \mathcal{H}^\mathcal{X}  } } \inf_{\lambda\in {\Lambda}_2^\mu(\bar{g},\bar{\boldsymbol{h}})} \sum_{x,g,h} \eta (x,g,h)\skl_{\mu|\lambda}(x,g,h),\\
											&= \min\limits_{\substack{\bar{g}\neq g^\star\\  \bar{\boldsymbol{h}}\in \mathcal{H}^\mathcal{X}  } }  \sum_{x\in \mathcal{X}} \inf_{\lambda\in {\Lambda}_2^\mu(\bar{g},\bar{\boldsymbol{h}})} \sum_{g,h} \eta (x,g,h)\skl_{\mu|\lambda}(x,g,h).
\end{align*}
The equation above stems from the fact that $(\bar g, \bar{\boldsymbol{h}})$ is fixed, for all the tasks. Using Lemma \ref{lemma:second_ineq_solution_inner_opt} we can replace the right-hand term in the above inequality as follows:
\begin{align*}
\skl(\delta_G, 1-\delta_G) &= \min\limits_{\substack{\bar{g}\neq g^\star\\  \bar{\boldsymbol{h}}\in \mathcal{H}^\mathcal{X}  } }  \sum_{x\in \mathcal{X}}\sum_{(g,h)\in \set{U}_{x,\bar g, \bar{h}_x}^\mu (\eta)} \eta(x,g,h) \skl(\mu(x,g,h), m(x,\bar g, \bar{h}_x;\eta,\mu)), \\
&=\min\limits_{\bar{g}\neq g^\star }  \sum_{x\in \mathcal{X}} \min_{\bar{h}_x\in \mathcal{H}}  \sum_{(g,h)\in \set{U}_{x,\bar g, \bar h_x}^\mu(\eta)} \eta(x,g,h) \skl(\mu(x,g,h), m(x,\bar g, \bar h_x;\eta,\mu)).
\end{align*}
The last equality holds because the quantity \[ \sum_{(g,h)\in \set{U}_{x,\bar g, \bar{\boldsymbol{h}}}^\mu(\eta)} \eta(x,g,h) \skl(\mu(x,g,h), m(x,\bar g, \bar h_x;\eta,\mu)),\] only depends on which $\bar{h}_x$ is chosen in task $x$. The proof is hence completed.
\end{proof}
\begin{lemma}\label{lemma:second_ineq_solution_inner_opt}
	For any $\eta \in \mathbb{R}_{\geq 0}^{\set X\times \set G \times \set H}$, for a fixed $x\in\mathcal{X}$ and an action $(\bar{g},\bar{\boldsymbol{h}})\in (\mathcal{G}\setminus\{g^\star\})\times \mathcal{H}^\mathcal{X}$,
	\[
	\inf_{\lambda \in {\Lambda}_2^\mu(\bar{g},\bar{\boldsymbol{h}})} \sum_{g,h} \eta(x,g,h) \skl_{\mu|\lambda}(x,g,h) = \sum_{(g,h)\in \set{U}_{x,\bar g, \bar h_x}^\mu(\eta)} \eta(x,g,h) \skl(\mu(x,g,h), m(x,\bar g, \bar h_x;\eta,\mu)),
	\]
	where $m(x,\bar g, \bar h_x;\eta,\mu)=\frac{\sum_{(g,h)\in \set U_{x,\bar g, \bar h_x}^\mu(\eta)} \eta(x,g,h)\mu(x,g,h)}{\sum_{(g,h)\in \set U_{x,\bar g, \bar h_x}^\mu(\eta)} \eta(x,g,h)},$ and $ \set{U}_{x,\bar g, \bar h_x}^\mu(\eta)$ is defined as
	\[
	\set{U}_{x,\bar g, \bar h_x}^\mu(\eta) \coloneqq \left\{(g',h') \in \mathcal{G}\times \mathcal{H}:  \mu(x,g',h') \geq \frac{\sum_{(g,h) \in \mathcal{N}_{x,g',h'}^\mu} \eta(x,g,h)\mu(x,g,h)} {\sum_{(g,h) \in \mathcal{N}_{x,g',h'}^\mu} \eta(x,g,h)} \right\}\cup \{\bar g, \bar h_x\}, 
	\]
	with $\mathcal{N}_{x,g,h}^\mu\coloneqq \{(g',h') \in \mathcal{G}\times \mathcal{H}: \mu(x,g',h')\geq \mu(x,g,h) \} \cup \{(\bar g, \bar h_x)\}$. 
\end{lemma}
\begin{proof}
Bear in mind that ${\Lambda}_2^\mu(\bar{g},\boldsymbol{\bar{h}})$  is the set of models where $\bar{g}$ is the optimal representation, and $\boldsymbol{\bar{h}}$ is the set of optimal predictors. Consequently, for a specific task $x$, we can rewrite  $	\inf_{\lambda \in {\Lambda}_2^\mu(\bar{g},\bar{\boldsymbol{h}})} \sum_{g,h} \eta(x,g,h) \skl_{\mu|\lambda}(x,g,h)$  as the solution of the following optimization problem:
\begin{equation*}
\begin{aligned}
\inf_{\lambda \in \set{M}}\quad & \sum_{g,h} \eta(x,g,h) \skl_{\mu|\lambda}(x,g,h),\\
\textrm{s.t.} \quad & \lambda(x,g,h) \leq \lambda(x, \bar{g}, \bar{h}_x),\forall (g,h)\neq (\bar g, \bar h_x).
\end{aligned}
\end{equation*}
Define $\mathcal{N}_{x,g,h}^\mu$ to be the set of arms whose mean is bigger than $\mu(x,g,h)$, and that also include $(\bar g, \bar h_x)$, i.e., $\mathcal{N}_{x,g,h}^\mu= \{(g',h') \in \mathcal{G}\times \mathcal{H}: \mu(x,g',h')\geq \mu(x,g,h) \} \cup \{(\bar g, \bar h_x)\}$. Then, one can directly verify that the solution of the previous optimization problem is given by
\[ 
 \begin{cases}
 \lambda(x,g,h)= m(x,\bar g, \bar h_x; \eta, \mu)&\hbox{ for } (g,h) \in \set{U}_{x,\bar g, \bar h_x}^\mu(\eta),\\
  \lambda(x,g,h)=\mu(x,g,h),&\hbox{ otherwise,}
 \end{cases} 
 \]
where
\[ 
m(x,\bar g, \bar h_x;\eta,\mu) =  \frac{\sum_{(g,h)\in \set U_{x,\bar g, \bar h_x}^\mu(\eta)} \eta(x,g,h)\mu(x,g,h)}{\sum_{(g,h)\in \set U_{x,\bar g, \bar h_x}^\mu(\eta)} \eta(x,g,h)},
\]
and the set $\set{U}_{x,\bar g, \bar h_x}^\mu(\eta)$ is defined as
\[
\set{U}_{x,\bar g, \bar h_x}^\mu(\eta) = \left\{(g',h') \in \mathcal{G}\times \mathcal{H}:  \mu(x,g',h') \geq \frac{\sum_{(g,h) \in \mathcal{N}_{x,g',h'}^\mu} \eta(x,g,h)\mu(x,g,h)} {\sum_{(g,h) \in \mathcal{N}_{x,g',h'}^\mu} \eta(x,g,h)} \right\} \cup \{\bar g, \bar h_x\}.
\] 
Observe that the set $\set{U}_{x,\bar g, \bar h_x}(\eta)$ is nonempty, since it includes $(\bar{g}, \bar{h}_x)$ and $(g^\star, h_x^\star)$. Therefore we have
\[
\inf_{\lambda \in {\Lambda}_2^\mu(\bar{g},\bar{\boldsymbol{h}})} \sum_{g,h} \eta(x,g,h) \skl_{\mu|\lambda}(x,g,h) = \sum_{(g,h)\in \set{U}_{x,\bar g, \bar h_x}^\mu(\eta)} \eta(x,g,h) \skl(\mu(x,g,h),  m(x,\bar g, \bar h_x;\eta,\mu)). 
\]
\end{proof}

\subsection{Proofs of Propositions \ref{pro:1} and \ref{pro:2}}

\begin{proof}[Proof of Proposition \ref{pro:1}]
Suppose $\eta^{(G)}$ ($\eta^{(H)}$ respectively) are vectors of $\mathbb{R}^{\mathcal{X}\times \mathcal{G}\times \mathcal{H}}_{\geq 0}$ that satisfies constraint (\ref{eq:cons1}) ((\ref{eq:cons2}) respectively). Since the KL-divergence is non-negative, $\eta^{(G)} + \eta^{(H)}$ satisfies both  constraints \ref{eq:cons1} and \ref{eq:cons2}. Consequently, we have that  $\eta^{(G)}+\eta^{(H)}\ge K^\star(\mu,\boldsymbol{\delta})$. We get that $K^\star(\mu,\delta_G)+K^\star(\mu,\delta_H)\ge  K^\star(\mu,\boldsymbol{\delta})$.
\end{proof}

\begin{proof}[Proof of Proposition \ref{pro:2}]
For ease of notation let $\set{U}_{x,\bar g, \bar h_x}=\set{U}_{x,\bar g, \bar h_x}^\mu(\eta)$ and $m(x,\bar g, \bar h_x)=m(x,\bar g, \bar h_x;\eta,\mu) $.
	From Lemma \ref{lemma:second_ineq_solution_inner_opt} and Theorem \ref{thm:full_lower_bound_offline} we have that
	\begin{align*}&\min_{\bar{\boldsymbol{h}}\in \mathcal{H}^\mathcal{X}  }  \inf_{\lambda\in {\Lambda}_2^\mu(\bar{g},\bar{\boldsymbol{h}})} \sum_{x,g,h} \eta (x,g,h)\skl_{\mu|\lambda}(x,g,h),\\
	=& \sum_{x\in \mathcal{X}} \min_{\bar{h}_x\in \mathcal{H}}  \sum_{(g,h)\in \set{U}_{x,\bar g, \bar h_x}} \eta(x,g,h) \skl(\mu(x,g,h), m(x,\bar g, \bar h_x)),\\
	\geq & \max_x \min_{\bar h_x}\sum_{(g,h)\in \set{U}_{x,\bar g, \bar h_x}} \eta(x,g,h) \skl(\mu(x,g,h), m(x,\bar g, \bar h_x)),
	\end{align*}
	where the last inequality stems the fact that we are considering a sum of non-negative terms. By using the fact that $ (x,g^\star,h_x^\star),(x,\bar{g},\bar{h})\in   \mathcal{U}_{x,\bar{g},\bar{h}}$, define $\set S=\{(x,g^\star,h_x^\star),(x,\bar{g},\bar{h})\}$, then
\begin{align*}
\sum_{(g,h)\in \set{U}_{x,\bar g, \bar h}} \eta(x,g,h) \skl(\mu(x,g,h), m(x,\bar g, \bar h))&\ge  \sum_{(g,h)\in \set S}\eta(x,g,h)\skl(\mu(x,g,h), m(x,\bar g, \bar h)),\\
&\geq \inf_{d\in [0,1]}\sum_{(g,h)\in \set S}\eta(x,g,h)\skl(\mu(x,g,h),d).
\end{align*}
Denote the minimizer by $\bar d$, then we obtain we obtain
\[\sum_{(g,h)\in \set S}\eta(x,g,h)\skl(\mu(x,g,h),\bar d)=(\eta(x,g^\star,h_x^\star)+\eta(x,\bar{g},\bar{h}))I_{\alpha_{x,\bar{g},\bar{h}_x}}(\mu(x,g^\star,h_x^\star),\mu(x,\bar{g},\bar{h})).\]
Therefore, using the previous result, together with the result from Theorem \ref{thm:full_lower_bound_offline}, we can derive that for all $\bar g\neq g^\star$ the following inequality holds
\[\skl(\delta_G, 1-\delta_G)\leq \max_x \min_{\bar h_x} (\eta(x,g^\star,h_x^\star)+\eta(x,\bar{g},\bar{h}))I_{\alpha_{x,\bar{g},\bar{h}_x}}(\mu(x,g^\star,h_x^\star),\mu(x,\bar{g},\bar{h})),\]
from which we get the statement of Proposition \ref{pro:2}.
\end{proof}


\newpage

\section{Sample complexity of \textsc{OSRL-SC}}

This section is devoted to the analysis of \textsc{OSRL-SC}. We first provide useful notations and then discuss the implementation of the algorithm in detail. In particular, we provide a full description of our stopping rule (not described in the main text due to space constraints). Finally, we analyse the sample complexity of \textsc{OSRL-SC}.

\subsection{Preliminaries}

We introduce some useful notation used in the remaining of the paper. Recall the following quantity from Theorem \ref{thm:full_lower_bound_offline}:
\begin{align*}
\inf_{\lambda \in \Lambda_2^\mu} \mathbb{E}[L_\tau]&=\inf_{\lambda \in \Lambda_2^\mu} \sum_{x,g,h}\eta(x,g,h) \skl_{\mu|\lambda}(x,g,h),\\
&=\min\limits_{\substack{\bar{g}\neq g^\star\\  \bar{\boldsymbol{h}}\in \mathcal{H}^\mathcal{X}  } }  \sum_{x\in \mathcal{X}} \inf_{\lambda\in {\Lambda}_2^\mu(\bar{g},\bar{\boldsymbol{h}})} \sum_{g,h} \eta (x,g,h)\skl_{\mu|\lambda}(x,g,h),\\&= \min_{\bar g\neq g^{\star}} \sum_{x}  \min_{\bar h_x \in \set H}\sum_{(g,h)\in \set{U}_{x,\bar g, \bar h_x}^\mu(\eta)} \eta(x,g,h) \skl(\mu(x,g,h), m(x,\bar g, \bar h_x;\eta,\mu)),\end{align*}
where $g^{\star}=g^{\star}(\mu)$. If one lets $q\in \Sigma$ be defined as $q(x,g,h)=\eta(x,g,h)/ \sum_{x,g,h}\eta(x,g,h)$ then one obtains the definition of $\rho(q,\mu)$:
\begin{equation}
\inf_{\lambda \in \Lambda_2^\mu} \sum_{x,g,h}\eta(x,g,h) \skl_{\mu|\lambda}(x,g,h) = \mathbb{E}[\tau]\underbrace{\inf_{\lambda \in \Lambda_2^\mu} \sum_{x,g,h}q(x,g,h) \skl_{\mu|\lambda}(x,g,h)}_{\coloneqq \rho(q,\mu)}.
\end{equation}
Therefore, we can rewrite $\rho(q,\mu)$  as follows
\begin{align*}
\rho(q,\mu)&= \inf_{\lambda \in \Lambda_2^\mu} \sum_{x,g,h}q(x,g,h) \skl_{\mu|\lambda}(x,g,h),\\
&= \min\limits_{\substack{\bar g\neq g^{\star}\\ \bar{\boldsymbol{h}}\in \set{H}^{\set X}}}\sum_x \inf_{\lambda \in {\Lambda}_2^\mu(\bar{g},\bar{\boldsymbol{h}})} \sum_{g,h} q(x,g,h) \skl_{\mu|\lambda}(x,g,h),\\
&\stackrel{(a)}{=}\min_{\bar g\neq g^{\star}} \sum_{x}  \min_{\bar h_x \in \set H}\sum_{(g,h)\in \set{U}_{x,\bar g, \bar h_x}^\mu(q)} q(x,g,h) \skl(\mu(x,g,h),  m(x,\bar g, \bar h_x;q, \mu)),
\end{align*}
where (a) follows from Lemma \ref{lemma:second_ineq_solution_inner_opt}. For simplicity, in the following, unless it is not clear from the context we will simply write $\set U_{x, \bar g, \bar h_x}$ instead of $\set U_{x, \bar g, \bar h_x}^{\mu}(\eta)$, for any $\mu \in \set M$, $\eta\in \mathbb{R}^{\set X\times \set G\times \set H}_{\geq 0}$. Similarly, we will write $m(x, \bar g, \bar h_x)$ instead of $m(x, \bar g, \bar h_x; \eta, \mu)$.

Now, let $\sigma>0$ and define the regularized version of $\rho$:
\[\rho_\sigma(q,\mu) \coloneqq  \rho(q,\mu) - \frac{1}{2\sigma}\|q\|_2^2, \quad \sigma \in \mathbb{R}_{>0}.\]
For a given $\mu\in \set M$, define  by $\rho_\sigma^{\star}$ and  $q_\sigma^{\star}$ respectively the maximum and the maximizer of $\rho_\sigma: q\to\mathbb{R}$:
\begin{gather*}
\rho_\sigma^\star(\mu) \coloneqq \max_{q\in \Sigma}\rho_\sigma(q,\mu),\quad
q_\sigma^\star(\mu) \coloneqq \argmax_{q\in \Sigma} \rho_\sigma(q,\mu),
\end{gather*}
where $\Sigma$ denotes the probability simplex over $\set X \times \set G\times \set H$. The  function $\rho_\sigma(\cdot, \mu)$ is strictly concave in $q$ for each $(\sigma, \mu)$. Therefore, for this function there is a unique maximizer that we denote by $q_\sigma^\star(\mu)$. We also know from Berge's Maximum Theorem \cite{berge1963topological} (check \cite{sundaram1996first} for a book reference) that $q_\sigma^\star(\mu)$ is a continuous function in $(\sigma, \mu)$, due to the strong concavity of $\rho_\sigma(\cdot,\mu)$. We define a time-dependent version of $\rho$ used by the algorithm: when $\hat{\mu}_t\in {\cal M}$, 
\begin{align}
\rho(q,\hat \mu_t)&= \inf_{\lambda \in \Lambda_2^{\hat \mu_t}} \sum_{x,g,h}q(x,g,h) \skl_{\mu|\lambda}(x,g,h),\nonumber\\
&= \min\limits_{\substack{\bar g\neq g_t^{\star}\\ \bar{\boldsymbol{h}}\in \set{H}^{\set X}}}\sum_x \inf_{\lambda \in {\Lambda}_2^{\hat{\mu}_t}(\bar{g},\bar{\boldsymbol{h}})} \sum_{g,h} q(x,g,h) \skl_{\mu|\lambda}(x,g,h),\nonumber\\
&=\min_{\bar g\neq g_t^{\star}} \sum_{x}  \min_{\bar h_x}\sum_{(g,h)\in \set{U}_{x,\bar g, \bar h_x}^{(t)}} q(x,g,h) \skl(\hat\mu_t(x,g,h),  m_t(x,\bar g, \bar h_x)),
\end{align}
where $g_t^{\star} = \argmax_{g} \hat \mu_t(g)$, $m_t(x,\bar g, \bar h_x)= m(x, \bar g, \bar h_x; q, \hat \mu_t)$ and $\set U_{x,\bar g, \bar h_x}^{(t)}=\set U_{x,\bar g, \bar h_x}^{\hat{\mu}_t}(q)$.

 Finally, we can define for any $\hat \mu_t \in \set M$ the function $\rho_\sigma(q,\hat \mu_t) = \rho(q,\hat \mu_t) - \frac{1}{2\sigma}\|q\|_2^2$ and the pair $\rho_\sigma^{\star}(\hat \mu_t) = \max_{q\in \Sigma} \rho_\sigma(q,\hat \mu_t),q_\sigma^{\star}(\hat \mu_t) =  \argmax_{q\in \Sigma} \rho_\sigma(q,\hat \mu_t)$. 
%

\subsection{Algorithmic details of \textsc{OSRL-SC}} \label{sec:algo_details_osrl_sc}
An essential part of the algorithm is computing the value of $q_\sigma^{\star}(\hat \mu_t)$ and $\rho_\sigma(\hat \mu_t)$. To this aim, we make use of the following lemma, which allows us to compute the aforementioned value by considering the dual problem of $\rho(q,\mu)$.
\begin{lemma}
For any $\mu\in \set M$, the  optimization problem:
\begin{equation}
\sup_{q\in \Sigma}\rho(q,\mu)=\sup_{q\in\Sigma}\min_{\bar g\neq g^\star} \sum_{x}  \min_{\bar h_x}\sum_{(g,h)\in \set{U}_{x,\bar g, \bar h}} q(x,g,h) \skl(\mu(x,g,h),  m(x,\bar g, \bar h_x)),
\end{equation}
 is equivalent to
\begin{equation}\sup_{q \in \Sigma}\inf\limits_{ \substack{\lambda\in \Delta(A)\\\theta\in \Delta(B)}} \sum_{\bar g\neq g^\star}\lambda_{\bar g}\sum_{x,\bar h} \theta_{x,\bar g,\bar h}\sum_{(g,h)\in \set{U}_{x,\bar g, \bar h}} q(x,g,h) \skl(\mu(x,g,h),  m(x,\bar g, \bar h)),\end{equation}
with $\Sigma$ being the probability simplex over $\set X\times \set G \times \set H$ and $\Delta(k), k\in \mathbb{N}$, being the probability simplex of dimension $k-1$, with $A=G-1$ and $B=HX(G-1)$.
\end{lemma}
\begin{proof}
 To ease the notation, let   \[f_{x,\bar g, \bar h}(q,\mu) \coloneqq\sum_{(g,h)\in \set{U}_{x,\bar g, \bar h}} q(x,g,h) \skl(\mu(x,g,h),  m(x,\bar g, \bar h)).\] One can proceed by reformulating the original problem in the following manner:
 \begin{equation*}
\begin{aligned}
\sup_{(q,z) \in \Sigma\times\mathbb{R}}z\quad \textrm{ subject to } \quad  z \leq \sum_{x} \min_{\bar h\in H} f_{x,\bar g, \bar h}(q,\mu)\quad  \forall \bar g \neq g^\star.
\end{aligned}
\end{equation*}
We can further expand the optimization problem and write
\begin{equation*}
\begin{aligned}
\sup_{(q,z,p) \in \Sigma\times\mathbb{R}\times \mathbb{R}^{G-1}}\quad & z\\
\textrm{s.t.} \quad &  z \leq \sum_{x}  p_{x,\bar g}\quad  \forall \bar g \neq g^\star,\\
&  p_{x,\bar g} \leq f_{x,\bar g,\bar h}(q,\mu)\quad  \forall h \in H.
\end{aligned}
\end{equation*}

The previous problem is a linear program with convex non-linear constraints. To notice this, remember from Lemma \ref{lemma:second_ineq_solution_inner_opt} that $f_{x,\bar g, \bar h}$ can be equivalently written as \[f_{x,\bar g, \bar h}(q,\mu)=	\inf_{\lambda \in {\Lambda}_2^\mu(\bar{g},\bar{\boldsymbol{h}})} \sum_{g,h} \eta(x,g,h) \skl_{\mu|\lambda}(x,g,h).\]
Thus $f_{x,\bar g, \bar h}(q,\mu)$ can be seen as a minimization problem over a set of concave functions, thus for a fixed $\mu$ we have that $f_{x,\bar g, \bar h}(q,\mu)$ is concave in $q$, and therefore $p_{x,\bar g} - f_{x,\bar g, \bar h}(q,\mu)$ is convex.

Now, denote by $L$ the Lagrangian function, and introduce the dual variables $\lambda,\nu$:
\begin{align*}
L(q,z,p,\lambda,\nu) &= z + \sum_{\bar g\neq g^\star} \lambda_{\bar g} \left(\sum_{x} p_{x,\bar g} -z\right) +\sum_{x,\bar g\neq g^\star}\sum_{\bar h} \nu_{x,\bar g,\bar h}(f_{x,\bar g,\bar h}(q,\mu)- p_{x,\bar g}),\\
&= z\left(1-\sum_{\bar g\neq g^\star} \lambda_{\bar g} \right)+ \sum_{x,\bar g\neq g^\star} \lambda_{\bar g} p_{x,\bar g} +\sum_{x,\bar g\neq g^\star}\sum_{\bar h} \nu_{x,\bar g,\bar h}(f_{x,\bar g,\bar h}(q,\mu)- p_{x,\bar g}),\\
&= z\left(1-\sum_{\bar g\neq g^\star} \lambda_{\bar g} \right)+ \sum_{x,\bar g\neq g^\star}  p_{x,\bar g}\left(\lambda_{\bar g}-\sum_{\bar h} \nu_{x,\bar g,\bar h}\right) +\sum_{x,\bar g\neq g^\star,\bar h} \nu_{x,\bar g,\bar h}f_{x,\bar g,\bar h}(q,\mu).
\end{align*}
We know that at the optimum $z^\star = \sup_{q,z,t} \inf_{(\lambda,\nu) \geq 0} L(q,z,t,\lambda,\nu)$. But, we can easily verify Slater's condition, therefore strong duality holds. Hence $z^\star=\inf_{(\lambda,\nu)\geq 0} \sup_{q,z,p} L(q,z,p,\lambda,\nu)$. But if $\sum_{\bar g\neq g^\star} \lambda_g \neq 1$ then the problem is unbounded, therefore it needs to be $\sum_{\bar g\neq g^\star} \lambda_{\bar g}=1$.  Similarly, we need $\lambda_{\bar g}=\sum_{\bar h} \nu_{x,\bar g,\bar h}$.  Introduce $\theta_{x,\bar g,\bar h} = \nu_{x,\bar g,\bar h}/\lambda_{\bar g}$. It follows that the dual problem is
\begin{equation*}
\sup_{q \in \Sigma}\inf\limits_{ \substack{\lambda\in \Delta(A)\\\theta\in \Delta(B)}}\sum_{\bar g\neq g^\star}\lambda_{\bar g}\sum_{x,\bar h} \theta_{x,\bar g,\bar h}f_{x,\bar g,\bar h}(q,\mu).
\end{equation*}
\end{proof}

We can directly solve the dual problem over the probability simplex by making use of standard optimization tools, such as Mirror Ascent. By using the KL-divergence as Bregman divergence this will result in the Entropic Descent algorithm \cite{beck2003mirror}. 

In order to use Mirror Ascent, we will make use of the fact that for a fixed $\mu\in \set M$ then $f_{x,\bar g, \bar h}(\cdot,\mu)$ is Lipschitz with respect to the $L^1$ norm, where the Lipschitz constant $L$ satisfies $L\geq \max_{i,j} \skl(\mu_i,\mu_j)$. This allows us to choose an appropriate learning rate $\eta_k$ for the computation of $q$: \[\eta_k = \frac{1}{L}\sqrt{\frac{2\ln(XGH)}{k}}, \quad k\geq 1.\]
Furthermore, we also have the following theoretical guarantee from Theorem 5.1 in \cite{beck2003mirror}:
\begin{theorem}[Theorem 5.1 in \cite{beck2003mirror}]\label{thm:mirror_ascent_guarantee}
Let $\{q_t\}_{t\geq 1}$ be a sequence of points generated by the mirror ascent algorithm, with a uniform starting point $q_1$. Then, for all $t\geq 1$ one has
\begin{equation}
	\max_{1\leq s\leq k} f_{x,\bar g, \bar h}(q_s,\mu) - \max_{q} f_{x,\bar g, \bar h}(q,\mu) \leq L \sqrt{\frac{2\ln(GHX)}{k}}.
\end{equation}
\end{theorem}

Now that ones knows which procedure can be used to compute $q^{\star}$, what is left is: (i) how to numerically compute the problem
\[f_{x,\bar g, \bar h}(q,\mu)=\inf_{\lambda \in {\Lambda}_2^\mu(\bar{g},\bar{\boldsymbol{h}})} \sum_{g,h} \eta(x,g,h) \skl_{\mu|\lambda}(x,g,h)\]
and (ii) also how to compute the gradient of $f_{x,\bar g, \bar h}(q,\mu)$ with respect to $q$.

The first problem (i) is taken care of by Lemma \ref{lemma:second_ineq_solution_inner_opt}. To solve (ii), we resort to libraries that implement automatic differentiation, such as PyTorch \cite{paszke2017automatic}. One needs to pay attention in doing so: rewrite $f_{x,\bar g,\bar h}$ as follows:
\[f_{x,\bar g, \bar h}(q,\mu) = \sum_{(g,h)\in \set{U}_{x,\bar g, \bar h_x}} \eta(x,g,h) \skl(\mu(x,g,h), m(x,\bar g, \bar h_x)).\]
One can observe that the series depends on the set $\set{U}_{x,\bar g, \bar h_x}$, which in turn depends on the variable $q$. Apart from this, one needs to pay attention to numerical errors when applying differentiable programming techniques. In the mirror ascent algorithm we clamp the gradient in $[-6,6]$ to prevent the exploding gradient effect.

Regarding the use of $\sigma$ in the algorithm, what we can say is that it is mostly needed only for theoretical purposes. The algorithm given an initial condition will always converge to some point, and we found no major differences by not having $\sigma$, or having a large value of $\sigma$, such as $\sigma \geq 10^5$. Moreover, for a bandit problem $\mu$ that admits a unique maximizer, we can say that it rarely happens that $\hat \mu_t$ has more than one maximizer.

Finally, denote by $c_{t}=(q,\lambda,\theta)_t$ the point found at the $t$-th step. We found that to speed-up the algorithm  at the $t+1$-th round one can use as initial condition $x_{1,t+1}$ the point $c_{t}=(q,\lambda,\theta)_{t}$ of the previous iteration. By doing so one may converge to a solution where some components of $q$, $\lambda$ or $\theta$ may be $0$. Furthermore, we do not have anymore the guarantees of Theorem \ref{thm:mirror_ascent_guarantee}. To avoid these two problems, we used as initial condition a convex combination between the previous point and a uniform point:
\begin{equation}
x_{1,t+1} = \alpha c_t + \frac{(1-\alpha)}{S} \mathbf{1},\quad \alpha \in [0,1],
\end{equation}
where $\mathbf{1}$ is a vector of ones, and $S=XGH+G-1+HX(G-1)$ is the dimensionality of a point $x$. We found out that a large value of $\alpha$, such as 	$\alpha \geq 0.5$, gave a computational speed-up improvement.

Last, but not least, notice that in the pseudocode of \textsc{OSRL-SC} \ref{alg:offline_algo_full} the computation of $\Psi_t(g)$ is undefined if $\hat \mu_t \notin \set M$. In that case, we will continue running the loop, until $\hat \mu_t\in \set M$, which happens almost surely (for a proof, please refer to  section E).
\subsection{Analysis of \textsc{OSRL-SC}}
\subsubsection{Analysis of the stopping rule}
In this subsection, we show how to devise a stopping rule that guarantees the $\delta_G$-PAC property for any sampling strategy. 

Taking inspiration from the classical work of Wald \cite{wald1948optimum} and Chernoff \cite{chernoff1959sequential} in sequential hypothesis testing, and the more recent work from Garivier and Kaufmann \cite{garivier2016optimal,garivier2019non,kaufmann2018mixture}, we now describe a stopping rule based on hypothesis testing, by expressly making use of a Generalized Likelihood Ratio Test to decide between a set of symmetric simple hypotheses. 

We shall now consider the decision problem of identifying the best representation $g^\star$: we have $G$ different symmetric simple hypotheses of which only one is true:
\[
H_1: (g^\star=g_1), \quad \dots\quad, H_{G}: (g^\star=g_{G}).
\]
These hypotheses play a symmetric role (in the sense that both type I and type II errors have the same importance), and one wants to decide for either one of the hypotheses. 

Now, we know that a certain decision $g^\star=g$ for a problem $\mu \in \set{M}$ is appropriate if $g^\star(\mu)=g$. Then, consider
 subsets $\Lambda(g)=\{\mu \in \set M: g^\star(\mu)=g\}$: it follows that $\set{M}=\cup_{g\in \mathcal{G}}\Lambda(g)$, and that a decision $g^\star=g$ is appropriate if $\mu \in \Lambda(g)$.  We are thus facing the equivalent problem of choosing amongst the following set of hypotheses $\{ {H}_j: (\mu \in\Lambda(g_j))\}_{j=1}^G$.
In order to decide between a finite set of symmetric hypotheses, it is common to consider a test on two non-overlapping and asymmetric hypotheses:
\[\tilde{{H}}_0:(\mu \in \set{M}\setminus \Lambda(g)) \hbox{ and } \tilde{{H}}_1: (\mu \in \Lambda(g)).\]
This allows us to define our Generalized log Likelihood Ratio statistics $\Psi_t(g)$ for each arm $g$ up to time $t$. Introduce $\boldsymbol{Y}_t=\{Y_s\}_{s\leq t}$ to be the collection of samples observed from all the arms up to time $t$, then the likelihood ratio at time $t$ for arm $g$ is
\begin{equation}
L_t(g)= \frac{\max_{\lambda \in \set{M}}\ell(\boldsymbol{Y}_t;\lambda)}{\max_{\lambda \in \set{M}\setminus \Lambda(g)} \ell(\boldsymbol{Y}_t;\lambda)},
\end{equation}
where $\ell(\boldsymbol{Y}_t;\lambda)$ represents the likelihood of having observed $\boldsymbol{Y}_t$ given a model $\lambda \in \set{M}$. Notice that large values of $L_t(g)$ tend to reject $\tilde{{H}}_0$, and when the arms belong to a canonical regular exponential distribution then we can write the log likelihood ratio as follows $\Psi_t(g) \coloneqq \ln L_t(g) $, where
\begin{equation}
\Psi_t(g)
=   \inf_{\lambda \in \set{M}\setminus \Lambda(g)} \ln\frac{\ell(\boldsymbol{Y}_t;\hat{\mu}_t)}{\ell(\boldsymbol{Y}_t;\lambda)} = \inf_{\lambda \in \set{M}\setminus \Lambda(g)} \sum_{x,g,h} N_{t}(x,g,h) \skl (\hat \mu_t(x,g,h), \lambda(x,g,h)),
\end{equation}
Define $
{\Lambda}_2 (\bar{g},\bar{\boldsymbol{h}})=	\{\lambda \in  \mathcal{M}: (g^\star,h_1^\star,\ldots,h_X^\star)(\lambda)=(\bar{g},\bar{\boldsymbol{h}})\}$. Then, for $\hat \mu_t\in \set M$, we can rewrite $\Psi_t$ in the following way
\begin{align}\label{eq:d21}
\Psi_t (\tilde{g})&=\inf_{\lambda \in \set{M}\setminus \Lambda(\tilde g)} \ln\frac{\ell(\boldsymbol{Y}_t;\hat{\mu}_t)}{\ell(\boldsymbol{Y}_t;\lambda)},\nonumber\\
&=\inf\limits_{ \substack{ ( \bar{g},\bar{\boldsymbol{h}})\in  ( \mathcal{G} \setminus \{\tilde{g}\} ) \times \mathcal{H}^\mathcal{X} \\ \lambda \in {\Lambda}_2 (\bar{g},\bar{\boldsymbol{h}}) }} \sum_{x,g,h}  N_t(x,g,h)\skl(\hat{\mu}_t(x,g,h),\lambda (x,g,h)),\nonumber\\
&= \min_{ \bar{g} \neq \tilde{g} }\min_{\boldsymbol{\bar h}\in \set{H}^{\set X}} \sum_{x}  \sum_{(g,h)\in \mathcal{U}^{(t)}_{x,\bar{g},\bar{h}_x}}N_t(x,g,h)\skl(\hat{\mu}_t(x,g,h),m_t(x,\bar{g},\bar{h}_x)),
\end{align}
where the last equation is an application of Lemma \ref{lemma:second_ineq_solution_inner_opt}. To ease the notation let $\bar m_t(x) \coloneqq m_t(x,\bar g, \bar h_x)$. Then, we can then define the stopping rule using a threshold function $\beta_t(\delta_G)$, with $\delta_G\in(0,1)$, as follows
\begin{align}
\tau_{G} &= \left\{t\in \mathbb{N}: \max_{\tilde g\in \mathcal{G}}\Psi_t(\tilde g)> \beta_t(\delta_G)\right\},\\
&=\left\{t\in \mathbb{N}: \max_{\tilde g\in \mathcal{G}} \inf_{\lambda \in \set{M}\setminus \Lambda(\tilde g)}  \sum_{x,g,h} N_{t}(x,g,h) \skl (\hat \mu_t(x,g,h), \lambda(x,g,h))> \beta_t(\delta_G)\right\},\\
&=\left\{t\in \mathbb{N}: \max_{\tilde  g\in \mathcal{G}} \min_{\bar g\neq \tilde  g} \sum_{x} \min_{\bar{h}_x} \sum_{(g,h)\in \set{U}^{(t)}_{x,\bar g, \bar h_x}} N_t(x,g,h) \skl(\hat \mu_t(x,g,h),  \bar m_t(x))> \beta_t(\delta_G) \right\}.
\end{align}
%
Notice, furthermore, that for the type of model considered the maximum is achieved for $\tilde g = g_t^\star$ whenever $\hat \mu_t\in \set M$.

To define the thresholds, we will make use of Theorem 14 in \cite{kaufmann2018mixture}, which is stated after the proof of Theorem \ref{theorem:deltapac}.
\subsubsection{Proofs of Theorem \ref{theorem:deltapac}}
\begin{proof}
	The proof can be seen as multi-task version of the Proposition 12 in \cite{garivier2016optimal} with an improved bound from \cite{kaufmann2018mixture}. Let $g_t^\star$ be the decision rule at time $t$, and for the sake of notation let $\bar m_t(x) \coloneqq m_t(x,\bar g, \bar h_x), \bar \mu(x) \coloneqq \mu(x, \bar g, \bar h_x)$. The probability of error can be bounded as follows:
 \begin{align*}
 &\mathbb{P}_\mu(\tau_{G} <\infty, \hat{g} \neq g^\star) = \mathbb{P}_\mu(\exists t \in \mathbb{N}: g_t^\star \neq g^\star, \Psi_t(g_t^\star) > \beta_t(\delta_G)),\\
 &\quad\stackrel{(a)}{\leq}
 \sum_{\tilde{g}\neq g^\star}\mathbb{P}_\mu(\exists t \in \mathbb{N}: g_t^\star =\tilde{g}, \Psi_t(\tilde g) > \beta_t(\delta_G)),\\
  &\quad=
    \sum_{\tilde{g}\neq g^\star}\mathbb{P}_\mu\left(\exists t \in \mathbb{N}:\min\limits_{\substack{\bar g\neq \tilde g\\\boldsymbol{\bar h}\in \set{H}^{\set X}}} \sum_{x}  \sum_{(g,h)\in \set{U}^{(t)}_{x,\bar g, \bar h_x}}  N_t(x,g,h)\skl(\hat{\mu}_t(x,g,h),  \bar m_t(x))  > \beta_t(\delta_G)\right),\\
    &\quad\stackrel{(b)}{\leq}
        \sum_{\tilde{g}\neq g^\star}\mathbb{P}_\mu\left(\exists t \in \mathbb{N}:\min\limits_{\substack{\bar g\neq \tilde g\\\boldsymbol{\bar h}\in \set{H}^{\set X}}} \sum_{x}  \sum_{(g,h)\in \set{U}^{(t)}_{x,\bar g, \bar h_x}}  N_t(x,g,h)\skl(\hat{\mu}_t(x,g,h),  \bar \mu(x))  > \beta_t(\delta_G)\right),\\
 &\quad\le \sum_{\tilde{g}\neq g^\star}\mathbb{P}_\mu\left(\exists t \in \mathbb{N}: \sum_{x,g,h}  N_t(x,g,h)\skl(\hat{\mu}_t(x,g,h),\mu (x,g,h)) > \beta_t(\delta_G) \right),\\
 &\quad\stackrel{(c)}{\leq} \sum_{\tilde g \neq g^\star} \exp\left(-\ln\left( \frac{G-1}{\delta_G} \right)\right)= (G-1) \frac{\delta_G}{G-1}=\delta_G.
 \end{align*}
 Inequality (a) follows by applying a union bound over the set $\set G\setminus\{g^\star\}$; (b) on the other hand is a consequence of the fact that $m_t$ minimizes the self-normalized sum. The last step (c) follows by applying Theorem \ref{thm:deviation_inequality} with $x=\ln\left(\frac{G-1}{\delta_G}\right)$ and $\set S = \set X \times \set G \times \set H$.
\end{proof}

\begin{theorem}[Theorem 14  in \cite{kaufmann2018mixture}]\label{thm:deviation_inequality}
Let $\phi:\mathbb{R}^+\to\mathbb{R}^+$ be the function defined by
\[\phi(x) = 2\tilde{p}_{3/2}\left(\frac{p^{-1}(1+x)+\ln(2\zeta(2))}{2}\right),\]
where $\zeta(s) = \sum_{n\geq 1} n^{-s}$, $p(u)=u-\ln(u)$ for $u\geq 1$ and for any $z\in [1,e]$ and $x\geq 0$:
\[
\tilde{p}_z(x) = \begin{cases}
e^{1/p^{-1}(x)}p^{-1}(x)\quad \hbox{ if } x \geq p^{-1}(1/\ln z),\\
z(x-\ln\ln z)\quad \hbox{otherwise.}
\end{cases}
\]
Then, for $\set{S}$ a subset of arms,
\[ \mathbb{P}\left(\exists t \in \mathbb{N}: \sum_{a\in \set{S}} N_t(a)\skl(\hat{\mu}_t(a), \mu(a)) -c\ln(d+\ln N_t(a)) \geq |\set{S}| \phi\left(\frac{x}{|\set{S}|}\right)\right)\leq e^{-x},\]
for some constants $c,d$. For Gaussian or Gamma distributions one can use $c=2,d=4$, while $c=3,d=1$ apply for one-dimensional exponential families.
\end{theorem}
\subsubsection{Analysis of the tracking rule}
We will first state the following lemma, that guarantees that our estimate $q_\sigma^\star(\hat \mu_t)$ converges to the true maximizer $q_\sigma^\star(\mu)$ in the limit:
\begin{lemma}\label{lemma:convergence_to_maximizer}
	Given a sampling strategy that guarantees $\mathbb{P}_\mu(\lim_{t\to\infty} \hat{\mu}_t =\mu)=1$, then, 
	\begin{equation}
	 \mathbb{P}\left(\lim_{t\to\infty}\rho_\sigma^\star(\hat\mu_t)=\rho_\sigma^\star(\mu)\right)=1,\qquad \mathbb{P}\left(\lim_{t\to\infty}q_\sigma^\star(\hat{\mu}_t)=q_\sigma^\star(\mu)\right)=1.
	\end{equation}
\end{lemma}
\begin{proof}
Given that $\rho_\sigma^\star(\mu), q_\sigma^\star(\mu)$ are continuous in $\set M$, the proof is an application of the Continuous mapping theorem to the sequences $\{\rho_\sigma^\star(\hat \mu_t)\}$ and$\{q_\sigma^\star(\hat \mu_t)\}$.
\end{proof}
Because of the forced exploration, the event $\set E=\{\lim_{t\to\infty}\hat \mu_t\to \mu\}$ is of probability $1$  by the law of large numbers. From Lemma \ref{lemma:convergence_to_maximizer} it follows that for $\varepsilon>0$ there exists $t_{\varepsilon}$ such that for $t>t_{\varepsilon}$ then $\|q_\sigma^\star(\hat \mu_t)-q_\sigma^\star(\mu)\|_\infty \leq\varepsilon$. It follows that $\{q_\sigma^\star(\hat \mu_t)\}_t$ is a converging sequence, which allows us to make use of the following result from Lemma 17 in \cite{garivier2016optimal}.
\begin{lemma}[Tracking lemma in \cite{garivier2016optimal}]\label{lemma:tracking_lemma}
For each $(x,g,h) \in \set X\times \set G\times \set H, N_t(x,g,h) > (\sqrt{t}-XGH/2)_+-1$. Furthermore, for all $\varepsilon>0$, for all $t_0\in \mathbb{N}$, there exists $t_\varepsilon\geq t_0$ such that
\[\sup_{t\geq t_0} \left\|q_\sigma(\hat \mu_t)- q_\sigma^\star(\mu)\right\|_\infty\leq \varepsilon \Rightarrow \sup_{t\geq t_\varepsilon} \left\|\frac{N_t}{t}-q_\sigma^\star(\mu)\right\|_\infty \leq 3(|X||G||H|-1)\varepsilon,\]
where $N_t \in \mathbb{N}^{\set X \times \set G\times \set H}$ denotes the vector of arm pulls.
\end{lemma}
It should be noted that \textsc{OSRL-SC} tracks the optimal allocation $q_\sigma^{\star}(\hat \mu_t)$ whenever $U_t=\emptyset$ \textbf{and} $\hat{\mu}_t\in \set M$. Such conditions guarantee that we end up sampling more than the lower bound provided in Lemma \ref{lemma:tracking_lemma}. A similar condition holds also for the Track and Stop algorithm proposed in \cite{garivier2016optimal}, where they need to have a unique optimal arm to track the optimal allocation, although not explicitely stated in the paper.

Finally, the previous lemma allows us to state the following result
\begin{lemma}\label{lemma:convergence_to_qsigma}
The sampling rule satisfies
\[
\mathbb{P}\left(\lim_{t\to\infty} \frac{N_t}{t}= q^\star_\sigma(\mu)\right)=1.
\]
\end{lemma}
\begin{proof}
 Because of the forced exploration, by the law of large numbers we have that $\mathbb{P}\left(\lim_{t\to\infty}\mu_t  =\mu\right)=1$. Let $\varepsilon>0$: then there exists $t_0$ such that for all $t\geq t_0$ the distance $\left\|q_{\sigma}^\star(\hat{\mu}_t)- q_{\sigma}^\star(\mu)\right\|_\infty \leq \varepsilon/3(XGH-1)$. Hence, using  Lemma \ref{lemma:tracking_lemma} we can conclude that there exists $t_\varepsilon\geq t_0$ such that for all $t\geq t_\varepsilon$ we have $\left\|\frac{N_t}{t}- q_\sigma^\star\right\|_\infty \leq \varepsilon$.
\end{proof}

Additionally, as mentioned previously in the experiment section (Appendix \ref{app:exp}), we will make use of the fact that $q^\star_\sigma(\hat \mu_t)$ is a sequence that convergence almost surely: this allows us to periodically compute it, instead of computing it at every round.
\begin{corollary}
Given a sampling strategy that guarantees $\hat{\mu}_t\to \mu$ for $t\to\infty$ almost surely, then, for any integer $k \in \mathbb{N}_{>0}$ we have that
\[\mathbb{P}\left(\lim_{t\to\infty} \rho_\sigma^\star(\hat \mu_{kt}) = \rho_\sigma(\mu)\right).\]
\end{corollary}
\begin{proof}
The proof is trivial, and follows from Lemma \ref{lemma:convergence_to_maximizer}. Due to the almost sure convergence of $\rho_\sigma^\star(\hat{\mu}_{t})$, then, the periodic subsequence $\{\phi(kt)\}_t$ is convergent, with $\phi(t)=\rho_\sigma^\star(\hat{\mu}_{t})$ and $k\in \mathbb{N}_{>0}$.
\end{proof}
Using this fact may save computational resources: furthermore, it will not influence the asymptotic analysis of sample complexity. But, despite that, it may affect the non-asymptotic sample complexity analysis. 
\subsubsection{Almost sure bound on the sample complexity of \textsc{OSRL-SC}}
Given the previous analysis of the tracking rule, we are now able to state a theorem that guarantees that $\{\tau_G<\infty\}$ happens with probability $1$, and also an almost sure upper bound on the asymptotic sample complexity, that follows directly from the proof in \cite{garivier2016optimal}.  

First, we state the following bound on the thresholds that \textsc{OSRL-SC} uses:
\begin{lemma}
Consider the threshold function defined in Theorem \ref{theorem:deltapac}:	
\begin{equation*}
	\begin{cases}
		\beta(t,\delta_G)=\beta_1(t) + \beta_2(\delta_G),\\
		\beta_1(t) = 3\sum_{x,g,h}\ln(1+\ln(N_t(x,g,h))),\quad \beta_2(\delta_G) = K\phi\left(\dfrac{\ln((G-1)/\delta_G)} {K}\right).
	\end{cases}
\end{equation*}
Then there exist constants $C,D>0$ such that
\[ \forall t\geq C, \beta(t,\delta_G)\leq \ln\left(\frac{Dt}{\delta_G}\right). \]
\end{lemma}
\begin{proof}
We can easily see that $\beta_1(t) \leq 3 \sum_{x,g,h} \ln(1+t) = 3XGH \ln(1+t)$, from which follows that $\beta_1(t) \leq 3XGH \ln(Bt)$, for $t \geq 1/(B-1)$, and $B>1$.

$\beta_2(\delta_G)$ on the other hand is just a constant term, and the function $\phi(x)\sim x$ at infinity, and can be upper bounded by a linear function in $x$. Since $x=\ln((G-1)/\delta_G)/K$ the claim follows.
\end{proof}
 We are now ready to state the almost sure bound on the sample complexity of \textsc{OSRL-SC}.
 
\begin{theorem}
For $\delta_G \in (0,1)$ let $\beta_t(\delta_G)$ be a deterministic sequence of thresholds that is increasing in $t$ and for which there exists constants $C,D>0$ such that
\[\forall t\geq C, \forall \delta \in (0,1), \beta_t(\delta_G) \leq \ln\left(\frac{D t}{\delta_G} \right).\]
Let $\tau_G$ be the sample complexity, with thresholds $\beta_t(\delta_G)$, Then, the tracking rule of \textsc{OSRL-SC} ensures, for $\sigma>0$:
\[\mathbb{P}_\mu\left(\limsup_{{\delta_G} \to 0}\frac{\tau_G}{K_{G,\sigma}^\star(\mu, \delta_G)} \leq 1\right)=1.\]
\end{theorem}
 \textit{Proof sketch: the proof relies on the continuity of $\rho_\sigma, q_\sigma^{\star}$ and the fact that $\hat{\mu}_t\to \mu$ and $N_t/t \to q_\sigma^{\star}(\mu)$ almost surely. This allows us to find a lower bound on $\max_g\Psi_t(g)$ that depends on the value of the problem, which in turns upper bounds the stopping time. Finally, a technical lemma is used to derive an upper bound.}
\begin{proof}
Because of the forced exploration step  the event \[\set E=\left\{\lim_{t\to\infty}\hat \mu_t\to \mu \hbox{ and } \lim_{t\to\infty} \frac{N_t}{t} =q_\sigma^{\star}(\mu))\right\},\] is of probability $1$. Remember also the continuity of $\rho_\sigma^{\star}(\mu)$ and $q_\sigma^{\star}(\mu)$ and the definition of $\set M(g) \coloneqq \{\mu \in \set M: g^\star(\mu)=g\}, \forall g\in \set G$. Therefore, given $\varepsilon>0$, on $\set E$ there exists $t_0$ such that for all $t\geq t_0$ we have (i) $\hat \mu_t\in \set M(g^\star)$ from which follows $\hat{\mu}_t(g^\star) > \max_{g\neq g^\star } \hat{\mu}_t(g)$ and (ii) $(1+\varepsilon)\rho_{\sigma}\left( \frac{N_t}{t},\hat\mu_t\right) \geq \rho_{\sigma}\left( q_{\sigma}^\star(\mu),\mu\right).$ Because of (i), for $t\geq t_0$ we are allowed to write
\begin{align*}
\max_g\Psi_t(g)&=\min_{\bar g\neq g^{\star}} \sum_{x} \min_{\bar{h}_x} \sum_{(g,h)\in \set{U}^{(t)}_{x,\bar g, \bar h_x}} N_t(x,g,h) \skl(\hat \mu_t(x,g,h),  m_t(x,\bar g, \bar h_x)),\\
&=t\min_{\bar g\neq g^{\star}} \sum_{x} \min_{\bar{h}_x} \sum_{(g,h)\in \set{U}^{(t)}_{x,\bar g, \bar h_x}} \frac{N_t(x,g,h)}{t} \skl(\hat \mu_t(x,g,h),  m_t(x,\bar g, \bar h_x)),\\
&=t\rho\left(\frac{N_t}{t},\hat\mu_t\right).
 \end{align*}
 On the other hand, because of (ii) we have for $t\geq t_0$:
\[t\rho\left(\frac{N_t}{t},\hat \mu_t\right)\stackrel{(a)}{\geq} t\rho_{\sigma}\left( \frac{N_t}{t},\hat\mu_t\right) \stackrel{(b)}{\geq}  \frac{t}{1+\varepsilon}\rho_{\sigma}\left( q_{\sigma}^\star(\mu),\mu\right), \]
where (a) follows from the definition of $\rho_\sigma$, since we are subtracting a positive value to $\rho$; (b) follows by continuity of $\rho_\sigma$ and $q_\sigma^{\star}$ and the fact that $t\geq t_0$.
Consequently, we can write
\begin{align}
\tau_G &= \inf \left\{t\in\mathbb{N}:\Psi_t(g^\star) \geq \beta_t(\delta_G) \right\},\\
&\leq \inf \left\{t \in \mathbb{N}: \frac{t}{1+\varepsilon} \rho_{\sigma}\left( q_{\sigma}^\star(\mu),\mu\right)\geq \beta_t(\delta_G)\right\},\\
&\leq  \inf \left\{t \in \mathbb{N}: \frac{t}{1+\varepsilon} \rho_{\sigma}\left( q_{\sigma}^\star(\mu),\mu\right)\geq \ln\left(\frac{Dt}{\delta_G}\right)\right\}.
\end{align}
By applying the technical Lemma \ref{lemma:log_inequality} with $\beta = D/\delta_G$, $\gamma=(1+\varepsilon)/\rho_{\sigma}\left( q_{\sigma}^\star(\mu),\mu\right)$ we get
\begin{align}
\tau_G \leq  \frac{1+\varepsilon}{\rho_{\sigma}\left( q_{\sigma}^\star(\mu),\mu\right)}\ln\left(\frac{D(1+\varepsilon)}{\delta_G \rho_{\sigma}\left( q_{\sigma}^\star(\mu),\mu\right)}\right) +  \sqrt{2\left(\ln\left(\frac{D(1+\varepsilon)}{\delta_G \rho_{\sigma}\left( q_{\sigma}^\star(\mu),\mu\right)}\right)-1\right)}.
\end{align}
Thus $\tau_G$ is finite on $\set E$ for every $\delta_G \in(0,1)$, and
\[\limsup_{\delta_G\to 0} \frac{\tau_G}{\skl(\delta_G,1-\delta_G)}\leq \frac{1+\varepsilon}{\rho_{\sigma}\left( q_{\sigma}^\star(\mu),\mu\right)} = (1+\varepsilon)C_{\sigma}(\mu).\]
Where $C_{\sigma}(\mu)=\rho_{\sigma}\left( q_{\sigma}^\star(\mu),\mu\right)^{-1}$. By letting $\varepsilon\to 0$ we obtain the desired result.
\end{proof}
\subsubsection{Bound on the expected sample complexity of \textsc{OSRL-SC}}
We prove our upper bound of the expected sample complexity of \textsc{OSRL-SC}. We first study the expected sample complexity of the first phase of \textsc{OSRL-SC}, and prove Theorem \ref{theorem:exp_sc_proof}. Then, to upper bound the expected duration of the second phase of the algorithm, we will use the results from \cite{garivier2016optimal}. Finally, we show that we can easily put the upper bounds for the first and second phases together, and deal with the fact that the first phase could output the wrong representation.   

{\bf Upper bound of the duration of the first phase.} 

\begin{theorem}\label{theorem:exp_sc_proof}
For $\delta_G \in (0,1)$ let $\beta_t(\delta_G)$ be a deterministic sequence of thresholds that is increasing in $t$ and for which there exists constants $C,D>0$ such that
\[\forall t\geq C, \forall \delta_G \in (0,1), \beta_t(\delta_G) \leq \ln\left(\frac{D t}{\delta_G} \right).\]
Let $\tau_G$ be the sample complexity with thresholds $\beta_t(\delta_G)$ Then, the tracking rule of \textsc{OSRL-SC} with $\sigma>0$ ensures
\[\limsup_{\delta_G\to 0} \frac{\mathbb{E}_\mu[\tau_G]}{K_{G,\sigma}^\star(\mu,\delta_G)} \leq 1.\]
\end{theorem}
 \textit{Sketch of the proof: The proof follows along the same lines of the results proved in \cite{garivier2016optimal} and \cite{kaufmann2018mixture}. The novelty is to adapt the definition of "good tail event" (see later in the proof) to our model, and derive the upper bound given a fixed value of $\sigma>0$. 
 }
\begin{proof}
 Recall the definition $g_t^\star = \argmax_g \hat \mu_t(g)$, and define
\begin{align*}\Delta &\coloneqq \frac{\max_x \mu(x,g^\star,h^\star_x) - \max_{x, (g,h)\neq(g^\star, h^\star_x)}\mu(x,g,h)}{4},\\ \set{I}_{\varepsilon} &\coloneqq \bigtimes_{x,g,h}[\mu(x,g,h)-\xi_\varepsilon(\mu), \mu(x,g,h)+\xi_\varepsilon(\mu)],\end{align*}
where $\varepsilon>0$. From the continuity of $q_\sigma^\star$ in $\mu$,  there exists $\xi_\varepsilon(\mu) \leq \Delta$ such that whenever $\hat{\mu}_t\in \set{I}_{\varepsilon}$ it holds that $g_t^\star =g^\star$, and $\|q_\sigma^\star(\hat \mu_t)-q^\star_\sigma(\mu)\|_\infty \leq \varepsilon$. Then, let $T\in \mathbb{N}$ and define the 'good' tail event
\[\set{E}_T(\varepsilon) = \bigcap_{t=T^{1/4}}^T \{\hat{\mu}_t \in \set{I}_{\varepsilon}\}.\]

Using Lemma \ref{lemma:tracking_lemma} and \ref{lemma:convergence_to_qsigma},
	there exists a constant $T_{\varepsilon}$ such that for $T\geq T_{\varepsilon}$ it holds that in $\set{E}_T(\varepsilon)$, for every $t\geq \sqrt{T}$, 
\[\sup_{t\geq t_{\varepsilon}} \left \|\frac{N_t}{t}-q_\sigma^\star(\mu)\right \|_\infty \leq 3(XGH-1)\varepsilon.\]
On the event $\set{E}_T(\varepsilon)$, for $t\geq T^{1/4}$, it holds that $g^\star_t=g^\star$. Consider

\[\rho\left(\frac{N_t}{t},\mu_t\right)=\min_{\bar g\neq g^{\star}} \sum_{x} \min_{\bar{h}_x} \sum_{(g,h)\in \set{U}^{(t)}_{x,\bar g, \bar h_x}} \frac{N_t(x,g,h)}{t} \skl(\hat \mu_t(x,g,h),  m_t(x,\bar g, \bar h_x)),\]
We know already that $\rho$ is continuous due to Berge's theorem, and as a corollary that $\rho_\sigma$ is continuous. Then, for $T\geq T_{\varepsilon},$ introduce the constant
\[C^\star_{\sigma,\varepsilon}(\mu) \coloneqq \inf\limits_{\substack{\mu'\in \set M:\|\mu-\mu'\|_1\leq \xi,\\q'\in\Sigma: \|q'-q_{\sigma}^\star(\mu)\|_\infty \leq 3(XGH-1)\varepsilon }} \rho_{\sigma}(q', \mu'). \]
 On the event $\set{E}_T(\varepsilon)$, it holds that, for every $t\geq \sqrt{T}$:
 \begin{align*}\max_{g\in \set G}  \Psi_t(g) &= t\underbrace{\min_{\bar g\neq g^{\star}} \sum_{x} \min_{\bar{h}_x} \sum_{(g,h)\in \set{U}^{(t)}_{x,\bar g, \bar h_x}} \frac{N_t(x,g,h)}{t} \skl(\hat \mu_t(x,g,h),  m_t(x,\bar g, \bar h_x))}_{\rho},\\
 &=t\rho\left(\frac{N_t}{t},\mu_t\right)\geq t\rho_{\sigma}\left(\frac{N_t}{t},\mu_t\right)\geq tC^\star_{\sigma,\varepsilon}(\mu).\end{align*}
 where the first inequality stems from the fact that $\rho(q,\mu)\geq \rho_\sigma(q,\mu)$ for every $ \sigma >0$.
 Let $T\geq T_\varepsilon$. On $\set{E}_T$,
 \begin{align*}
 \min(\tau_G, T) &\leq \sqrt{T} + \sum_{t=\sqrt{T}}^T \indi_{\{\tau_G > t\}}\leq \sqrt{T} + \sum_{t=\sqrt{T}}^T \indi_{\{ \max_{g\in G}\Psi_t(g)\leq \beta(t,\delta_G)\}},\\
  &\leq \sqrt{T} + \sum_{t=\sqrt{T}}^T \indi_{\{ tC^\star_{\sigma,\varepsilon}(\mu)\leq \beta(T,\delta_G)\}} \leq \sqrt{T} +  \frac{\beta(T,\delta_G)}{C^\star_{\sigma,\varepsilon}(\mu)}.
 \end{align*}
 Now, let $T_0^\varepsilon(\delta_G) = \inf\left \{T\in\mathbb{N}: \sqrt{T} +   \frac{\beta(T,\delta_G)}{C^\star_{\sigma,\varepsilon}(\mu)} \leq T\right \}$. 
And consider the following lemma (which is proved later on)
 \begin{lemma}\label{lem18}
 	There exist two constants $B_\varepsilon(\mu),K_\varepsilon(\mu)$ (that depend on $\mu$ and $\varepsilon$) such that
 	\[\mathbb{P}_\mu(\set{E}^c_T(\varepsilon)) \leq  B_\varepsilon(\mu)Te^{-K_\varepsilon(\mu)T^{1/8}}.	\]
 \end{lemma}
 
 Then, for every $T\geq \max(T_0(\delta_G), T_\varepsilon)$ we have
 \[P_\mu(\tau_G > T) \leq P_\mu(\set{E}_T^c(\varepsilon)) \leq B_\varepsilon(\mu)Te^{-K_\varepsilon(\mu)T^{ 	1/8}},\]
and therefore
\[\mathbb{E}_\mu[\tau_G]\leq T_0^\varepsilon(\delta_G)+T_\varepsilon+\sum_{T\geq 1} B_\varepsilon(\mu)Te^{-K_\varepsilon(\mu)T^{1/8}}. \]

What is left is to derive an upper bound of $T_0^\varepsilon(\delta_G)$. Let $\zeta>0$ and introduce the following constant
\[C(\zeta) = \inf\left\{T\in\mathbb{N}: \sqrt{T} +\frac{T}{1+\zeta}\leq T \right \}.\]
Remember that we use a deterministic sequence of thresholds that is increasing in $t$ and for which there exist constants $C,D>0$ such that
\[\forall t\geq C, \forall \delta_G \in (0,1), \beta_t(\delta_G) \leq \ln\left(\frac{D t}{\delta_G} \right).\]
We can derive an upper bound of $T_0^\varepsilon(\delta_G)$ by observing that $T_0^\varepsilon(\delta_G)$ is defined by  the quantity $\frac{\beta(T,\delta_G)}{C^\star_{\sigma,\varepsilon}(\mu)} \leq T-\sqrt{T}$. Since $C(\zeta)$ takes care of the case where $T-\sqrt{T} \geq T/(1+\zeta)$, we can then upper bound $\frac{\beta(T,\delta_G)}{C^\star_{\sigma, \varepsilon}(\mu)}$ by considering $T/(1+\zeta)$:
\[T_0^\varepsilon(\delta_G)\leq C+C(\zeta)+\inf\left\{T\in\mathbb{N}: C_{\sigma,\varepsilon}^\star(\mu)^{-1}\ln(DT/\delta_G)\leq \frac{T}{1+\zeta}\right\}.  \]

By applying Lemma \ref{lemma:log_inequality} with $\beta = D/\delta_G$, $\gamma=(1+\zeta)/C_{\sigma, \varepsilon}^\star(\mu)$, we get
\[T_0^\varepsilon(\delta_g) \leq C+C(\zeta)+  \frac{1+\zeta}{C_{\sigma,\varepsilon}^\star(\mu)}\left[\ln\left(\frac{D(1+\zeta)}{\delta_G C_{\sigma,\varepsilon}^\star(\mu)}\right) +  \sqrt{2\left(\ln\left(\frac{D(1+\zeta)}{\delta_G C_{\sigma,\varepsilon}^\star(\mu)}\right)-1\right)}\right]. \]
%
Finally, consider 
\[\frac{\mathbb{E}_\mu[\tau_G]}{\log(1/\delta_G)}\leq \frac{T_0^\varepsilon(\delta_G)}{\log(1/\delta_G)}+\frac{T_\varepsilon}{{\log(1/\delta_G)}}+\frac{1}{{\log(1/\delta_G)}}\sum_{T\geq 1} BTe^{-KT^{1/8}}. \]
Then, clearly for every $\varepsilon>0$ and $\zeta>0$ we get
\[\limsup_{\delta_G\to 0}  \frac{\mathbb{E}_\mu[\tau_G]}{\log(1/\delta_G)} \leq \limsup_{\delta_G\to 0}  \frac{T_0^\varepsilon(\delta_G)}{\log(1/\delta_G)}=\frac{1+\zeta}{C_{\sigma,\varepsilon}^\star(\mu)}. \]
Now, let $\varepsilon$ and $\zeta$ go to $0$, and by continuity of $\rho_\sigma$ and $q_{\sigma}^\star$:
\[\lim_{\varepsilon\to 0}C_{\sigma,\varepsilon}^\star(\mu)= \rho_\sigma(q_\sigma^{\star}(\mu),\mu),\]
and, finally
\[\limsup_{\delta_G\to 0}  \frac{\mathbb{E}_\mu[\tau_G]}{K_{G,\sigma}^\star(\mu,\delta_G)} \leq 1. \]
\end{proof}

{\bf Proof of Theorem \ref{theorem:exp_sc_bound}: Upper bound of the expected sample complexity of OSRL-SC.}

The result in Theorem \ref{theorem:exp_sc_bound} follows from those of Theorem \ref{theorem:exp_sc_proof} and Theorem 14 in \cite{garivier2016optimal}. The latter result upper bounds the expected duration of the second phase of OSRL-SC if we use, for this phase, a threshold rule $\beta_t(\delta_H)$ that is increasing in $t$ and for which there exists constants $C,D>0$ such that 
\[\forall t \geq C, \forall \delta_H \in (0,1), \beta_t(\delta_H) \leq \ln \left(\frac{Dt}{\delta_H}\right). \]
For example, in the Bernoulli case, one can choose $\beta_t(\delta_H)= \log\left(\frac{2t(H-1)}{\delta_H}\right)$.

To derive Theorem \ref{theorem:exp_sc_bound}, observe that the sample complexity of the second phase can be rewritten as follows \[\mathbb{E}_\mu[\tau_H] = \mathbb{E}_\mu[\tau_H|\hat g=g^\star]\mathbb{P}_{\mu}(\hat g = g^\star)  + \mathbb{E}_\mu[\tau_H|\hat g\neq g^\star]\mathbb{P}_{\mu}(\hat g \neq g^\star),\]
where $\mathbb{E}_\mu[\tau_H| \hat g =g]$ denotes the conditional expected sample complexity of the second phase, given that the first phase outputs $\hat g$. From the result of \cite{garivier2016optimal}, we know that
\[
\lim\sup_{\delta_H\to0} {\mathbb{E}_\mu[\tau_H|\hat g=g^\star]\over K_H^\star(\mu, \delta_H;g^\star)}\le 1,
\]
where  $K_H^\star(\mu, \delta_H;g^\star) \coloneqq  \sum_{x} T^\star(x,g^\star;\mu )\skl(\delta_H,1-\delta_H)$, and $T^\star(x, g^\star)$ is defined as
\[T^\star(x,g^\star;\mu)^{-1} = \sup_{q\in \Sigma} \min_{h\neq \hat h_x^{\star}} (\eta (x,g^\star,\hat h_x^{\star})+\eta (x,g^\star,h) )I_{\alpha_{x,g^\star,h}}(\mu(x,g^\star,\hat h_x^{\star}), \mu(x,g^\star,h)),\]
with  $\hat h_x^\star = \argmax_h \mu(x,g, h)$.  For $g\neq g^\star$ we instead have
\[\lim\sup_{\delta_H\to0} {\mathbb{E}_\mu[\tau_H|\hat g=g]\over \skl(\delta_H,1-\delta_H)}\le C,\]
for some positive constant $C$ that depends on the threshold rule $\beta_t(\delta_H)$. 
Since the first phase of the algorithm is $\delta_G$-PAC, we have:
\[ \mathbb{E}_\mu[\tau_H] \leq  \mathbb{E}_\mu[\tau_H|\hat g=g^\star]\mathbb{P}_{\mu}(\hat g = g^\star)  + \mathbb{E}_\mu[\tau_H|\hat g\neq g^\star] \delta_G.\]
Hence, we get: for any $\delta_G$, 
\[
\lim\sup_{\delta_H\to0} {\mathbb{E}_\mu[\tau_H]\over K_H^\star(\mu, \delta_H;g^\star)\mathbb{P}_{\mu}(\hat g = g^\star) +\delta_G C(G-1) \skl(\delta_H,1-\delta_H)}\le 1.
\]
Letting $\delta_G\to 0$, we easily obtain the statement of Theorem \ref{theorem:exp_sc_bound} :
\[\limsup_{\delta_H,\delta_G\to 0} \frac{\mathbb{E}_\mu[\tau]}{K^\star_{G,\sigma}(\mu, \delta_G) +K_H^\star(\mu, \delta_H)}\leq 1.\]
\ep
 
 \medskip
{\bf Proofs of the remaining lemmas.}

Lemma \ref{lem18} was used to bound the probability of not being in the "\textit{good set}" $\set E_T(\varepsilon)$. It is similar to Lemma 19 in \cite{garivier2016optimal}. We provide its proof for completeness.

\textit{Proof of Lemma \ref{lem18}.} An explicit expression of $\mathbb{P}_\mu(\hat{\mu}_t\notin \set{I}_{\varepsilon})$ will allow to apply a Chernoff bound, and derive the desired result.
	Let $h(t)=t^{1/4}$, then
	\begin{align*}
	\mathbb{P}_\mu&(\set{E}^c_T(\varepsilon)) \leq \sum_{t=h(T)}^T \mathbb{P}_\mu(\hat{\mu}_t\notin \set{I}_{\varepsilon}),\\
	&\quad\leq \sum_{t=h(T)}^T\sum_{x,g,h} \mathbb{P}_\mu(\hat{\mu}_t(x,g,h)\notin [\mu(x,g,h)-\xi_\varepsilon(\mu), \mu(x,g,h)+\xi_\varepsilon(\mu)]),\\
	&\quad\leq \sum_{t=h(T)}^T\sum_{x,g,h} \mathbb{P}_\mu(\hat{\mu}_t(x,g,h) \leq  \mu(x,g,h)-\xi_\varepsilon(\mu))  +\mathbb{P}_\mu(\hat{\mu}_t(x,g,h) \geq  \mu(x,g,h)+\xi_\varepsilon(\mu)).
	\end{align*}
	Now, let $K=XGH$ and $T$ such that $h(T)\geq K^2$. Then for $t\geq T$ one has $N_t(x,g,h) \geq (\sqrt{t}-K/2)_+-1\geq \sqrt{t}-K$ for every arm $(x,g,h)$, because of Lemma \ref{lemma:tracking_lemma}. Finally, let $\hat{\mu}_{x,g,h,s}$ be the empirical mean of the first $s$ rewards from arm $(x,g,h)$. Then
	\begin{align*}
	\mathbb{P}_\mu& \left(  \hat{\mu}_t(x,g, h)-\mu(x,g, h) \geq  \xi_\varepsilon(\mu)\right)\\&\qquad\qquad =\mathbb{P}_\mu \left(  \hat{\mu}_t(x,g, h)-\mu(x,g, h) \geq  \xi_\varepsilon(\mu), N_t(x,g,h) \geq \sqrt{t}-K\right),\\
	&\qquad\qquad \stackrel{(a)}{\leq} \sum_{s=\sqrt{t}-K}^t \mathbb{P}_\mu \left(  \hat{\mu}_{x,g, h, s}-\mu(x,g, h) \geq  \xi_\varepsilon(\mu)\right),\\
	&\qquad\qquad\stackrel{(b)}{\leq} \sum_{s=\sqrt{t}-K}^t \exp(-s \skl(\mu(x,g, h)+\xi_\varepsilon(\mu), \mu(x,g, h))),\\
	&\qquad\qquad\leq \sum_{s=0}^{t-\sqrt{t}+K} \exp(-(\sqrt{t}-K+s) \skl(\mu(x,g, h)+\xi_\varepsilon(\mu), \mu(x,g, h))),\\
	&\qquad\qquad\stackrel{(c)}{\leq} \frac{e^{-(\sqrt{t}-K) \skl(\mu(x,g, h)+\xi_\varepsilon(\mu), \mu(x,g, h))}}{1-e^{-\skl(\mu(x,g, h)+\xi_\varepsilon(\mu), \mu(x,g, h))}},
	\end{align*}
	where (a) follows from a union bound over $s\in [\sqrt{t}-K, t]$; (b) is an application of Chernoff bound; (c) follows by bounding the geometric series by $\left( 1-e^{-\skl(\mu(x,g, h)+\xi_\varepsilon(\mu), \mu(x,g, h))}\right)^{-1}$. 
	To ease the notation, define \begin{align*}K_\varepsilon(x,g;\mu) &= \min_{ h} (\skl(\mu(x,g, h)-\xi_\varepsilon(\mu), \mu(x,g, h)) \wedge \skl(\mu(x,g, h)+\xi_\varepsilon(\mu), \mu(x,g, h))),\\ K_\varepsilon(\mu)&=\min_{x,g} K_\varepsilon(x,g;\mu),\end{align*} and 
	\[B_\varepsilon(x,g;\mu)=\sum_{{h}}  \frac{e^{K\skl(\mu(x,g, h)-\xi_\varepsilon(\mu), \mu(x,g, h))}}{1-e^{-\skl(\mu(x,g, h)-\xi_\varepsilon(\mu), \mu(x,g, h))}} +  \frac{e^{K\skl(\mu(x,g, h)+\xi_\varepsilon(\mu), \mu(x,g, h))}}{1-e^{-\skl(\mu(x,g, h)+\xi_\varepsilon(\mu), \mu(x,g, h))}}   \]
	with $B_\varepsilon(\mu)=\sum_{x,g}B_\varepsilon(x,g;\mu)$. Now we can wrap up and derive the desired result:
	\begin{align*}
	\sum_{t=h(T)}^T \mathbb{P}_\mu(\hat{\mu}_t\notin \set{I}_{\varepsilon})&\leq \sum_{t=h(T)}^T\sum_{x,g,h} \mathbb{P}_\mu(\hat{\mu}_t(x,g,h)\notin [\mu(x,g,h)-\xi_\varepsilon(\mu), \mu(x,g,h)+\xi_\varepsilon(\mu)]),\\
	&\leq \sum_{t=h(T)}^T\sum_{x,g}  B_\varepsilon(x,g;\mu) e^{-K_\varepsilon(x,g;\mu)\sqrt{t}}\leq \sum_{t=h(T)}^T\sum_{x,g}  B_\varepsilon(x,g;\mu) e^{-K_\varepsilon(\mu)\sqrt{t}},\\
	&\leq \sum_{t=h(T)}^T B_\varepsilon(\mu)e^{-K_\varepsilon(\mu)\sqrt{t}} \leq B_\varepsilon(\mu)Te^{-K_\varepsilon(\mu)T^{1/8}}.
	\end{align*}
\ep

\medskip
Finally, we state and prove the following technical lemma. 

\begin{lemma}\label{lemma:log_inequality}
For every $\beta, \gamma>0$ such that $\beta\gamma >e$ and $ x\geq1$. If $\ln(\beta x) \leq \frac{x}{\gamma}$, then \[x \leq \gamma \left[\ln\left(\beta\gamma\right) +\sqrt{2\left(\ln\left(\beta \gamma\right)-1\right)}\right].\]
\end{lemma}

\begin{proof}
Let  $h(x)=x-\ln(x)$: it is an increasing function on $[1,\infty)$, therefore its inverse exists. $h^{-1}(x)$, defined on $[1,\infty]$, satisfies $h^{-1}(x) = -W_{-1}(-e^{-x})$, where $W_{-1}$ is the negative branch of the Lambert function.
A property of the function $W$ is that the solution to equations of the kind $x-be^{cx}=a$ is given by $x=a-W(-bce^{ac})/c$. Now we have: 
\[\ln(\beta x)\leq \frac{x}{\gamma} \Rightarrow   x \leq \frac{1}{\beta}\exp\left(\frac{x}{\gamma}\right). \]
Identify $b=1/\beta, c=1/\gamma$ and $a=0$. We obtain:
\[x\leq -\gamma W_{-1}\left(-\frac{1}{\beta\gamma}\right) =-\gamma W_{-1}\left(-\exp\left(-\ln\left(\beta \gamma\right)\right)\right). \]
We finally use the following inequality from Theorem 1 in \cite{chatzigeorgiou2013bounds}: for any  $x>1$,
$-W_{-1}\left(-e^{-x}\right) \leq x+\sqrt{2(x-1)}$.  
Hence, we have: $x \leq \gamma \left(\ln(\beta \gamma) + \sqrt{2\left(\ln(\beta\gamma) - 1\right)}\right).$
\end{proof}

\bigskip
\section{Experimental setup}
\paragraph{Hardware and software setup.} All experiments were executed on a stationary desktop computer, featuring an Intel Xeon Silver 4110 CPU, 48GB of RAM. Ubuntu 18.04 was installed on the computer. Ubuntu is an open-source Operating System using the Linux kernel and based on Debian. For more information, please check \url{https://ubuntu.com/}.

\paragraph{Code and libraries.} We set up our experiments using Python 3.7.7 \cite{van1995python} (For more information, please refer to the following link \url{http://www.python.org}), and made use of the following libraries: Cython version 0.29.15 \cite{behnel2011cython}, NumPy version 1.18.1 \cite{oliphant2006guide}, SciPy version 1.4.1 \cite{2020SciPy-NMeth}, PyTorch version 1.4.0 \cite{paszke2017automatic}. In the code, we make use of the KL-UCB algorithm \cite{garivier2011kl} and the Track and Stop algorithm \cite{garivier2016optimal}, published with MIT license. All the code can be found here \url{https://github.com/rssalessio/OSRL-SC}.

To run the code, please, read the attached README file for instructions on how to run experiments using Jupyter Notebook.
\newpage

\end{document}